\documentclass{article}

\PassOptionsToPackage{dvipsnames}{xcolor}
\usepackage[preprint]{neurips_2023}

\usepackage[dvipsnames]{xcolor}
\usepackage{microtype}
\usepackage{graphicx}
\usepackage{booktabs} 

\definecolor{darkblue}{rgb}{0, 0, 0.5}

\usepackage[utf8]{inputenc} 
\usepackage[T1]{fontenc}    
\usepackage{hyperref}       
\usepackage{url}            
\usepackage{booktabs}       
\usepackage{amsfonts}       
\usepackage{nicefrac}       
\usepackage{microtype}      
\usepackage{tablefootnote}

\hypersetup{colorlinks=true, citecolor=darkblue, linkcolor=darkblue, urlcolor=darkblue}

\usepackage[dvipsnames]{xcolor}         
\usepackage{comment}
\usepackage{amsmath}
\usepackage{amssymb}
\usepackage{bm}
\usepackage{hyperref}
\usepackage{amsthm}
\usepackage{amsopn}
\usepackage{graphicx}
\usepackage{caption}
\usepackage{subcaption}
\usepackage{wrapfig}
\usepackage{enumitem}
\usepackage{soul}
\usepackage{etoolbox}
\usepackage{algorithm}
\usepackage{algpseudocode}
\usepackage{arydshln}

\usepackage[normalem]{ulem}
\useunder{\uline}{\ul}{}

\DeclareMathOperator{\Attn}{Attn}

\theoremstyle{plain}
\newtheorem{theorem}{Theorem}[section]
\newtheorem{proposition}[theorem]{Proposition}
\newtheorem{lemma}[theorem]{Lemma}

\theoremstyle{definition}
\newtheorem{definition}[theorem]{Definition}
\newtheorem{assumption}[theorem]{Assumption}
\theoremstyle{remark}
\newtheorem{remark}[theorem]{Remark}

\BeforeBeginEnvironment{definition}{\vspace{-1pt}} 
\AfterEndEnvironment{definition}{\vspace{-3pt}}   
\BeforeBeginEnvironment{assumption}{\vspace{-1pt}} 
\AfterEndEnvironment{assumption}{\vspace{-3pt}}   
\BeforeBeginEnvironment{theorem}{\vspace{-1pt}} 
\AfterEndEnvironment{theorem}{\vspace{-3pt}}   
\BeforeBeginEnvironment{proposition}{\vspace{-1pt}} 
\AfterEndEnvironment{proposition}{\vspace{-3pt}}   

\definecolor{mycolor1}{RGB}{84, 130, 53}
\definecolor{mycolor2}{RGB}{112, 48, 160}
\definecolor{mycolor3}{RGB}{132, 60, 12}
\definecolor{mycolor4}{RGB}{47, 85, 151}
\definecolor{newyellow}{RGB}{255, 255, 0}
\definecolor{lightblue}{RGB}{0, 255, 255}

\begin{document}

\title{Parrot Mind: Towards Explaining the Complex Task \\Reasoning of Pretrained Large Language Models with Template-Content Structure}






\author{%
  Haotong Yang$^{12}$\quad Fanxu Meng$^{1}$\quad Zhouchen Lin$^{123}$\thanks{Corresponding authors} \quad Muhan Zhang$^{1}$\footnotemark[1]\\
  $^1$Institution for Artificial Intelligence, Peking University\\
  $^2$Key Lab of Machine Perception (MoE),\\ School of Intelligence Science and Technology, Peking University\\
  $^3$Peng Cheng Laboratory\\
  \texttt{haotongyang@pku.edu.cn\quad fxmeng@stu.pku.edu.cn\quad \{zlin,muhan\}@pku.edu.cn}
  }


\newcommand{\fix}{\marginpar{FIX}}
\newcommand{\new}{\marginpar{NEW}}

%


\maketitle

\begin{abstract}
The pre-trained large language models (LLMs) have shown their extraordinary capacity to solve reasoning tasks, even on tasks that require a complex process involving multiple sub-steps. However, given the vast possible generation space of all the tasks, how the pretrained model learns the reasoning ability remains an open question.
We firstly propose that an \textit{intrinsic structural constraint} on the generated sequence of language-based reasoning --- we called it \textbf{template-content structure} (\textbf{T-C structure}) --- is the key to explain why LLMs can solve a large number of complex reasoning problems with limited training data by showing this structure can reduce the possible space from \textit{exponential level} to \textit{linear level}.
Furthermore, by generalizing this structure to the hierarchical case, we demonstrate that models can achieve task composition, further reducing the space needed to learn from linear to \textit{logarithmic}, thereby effectively learning on complex reasoning involving multiple steps. We provide both examples and formal theory of our T-C structure. We also experimentally validate the existence of the T-C structure in some current LLMs and its effectiveness for reasoning.

\end{abstract}
\section{Introduction}
\label{sec:intro}

The continuous evolution of pre-trained Large Language Models (LLMs)~\citep{brown2020language, chowdhery2022palm, openai2023gpt4} with ever-growing parameters and corpus sizes has notably augmented their capacity to solve various complex tasks in natural language without fine-tuning these tasks. These language-based reasoning ability ranges from arithmetic and symbolic logic to factual reasoning~\citep{qin2023chatgpt, liu2023evaluating, Yang_2022_empirical, bang2023multitask, tan2023evaluation}. 
Compared to the diversity and complexity of real-world problems, the current LLM training paradigm is extremely simple. The training goal is only to simulate the probability distribution of the next word given a partial sequence, which is called language modeling~\citep{peters-etal-2018-deep, radford2018improving}\footnote{In this paper, we use the term ``language modeling'' in a more general way to include some later stages of pretraining such as instruction tuning, code tuning and RLHF, because these training strategy can be seen as language modeling on specially-designed corpus}. A long-standing debate about whether LLMs can understand and reason is rooted in this mismatch. This doubt can be summarized as: \textit{whether LLMs are just \textbf{parroting} the training samples it has seen before}. 

A series of studies have explored the question of LLMs' capabilities both before and after they demonstrated their remarkable performance across a wide range of tasks. On the one hand, some researchers have attempted to highlight the significant dependence of LLMs on their training data, similar to a parrot mindlessly mimicking \textbf{without a deeper understanding} of the underlying content. As noted by \citet{Bender2021dangers}, earlier LLMs behaved as \textit{stochastic parrots}, essentially replicating the patterns found in their training data. \citet{zevcevic2023causalparrots} further argue that the apparent ability of certain current LLMs to do causal inference is primarily a result of memorizing and regurgitating pretrained causal knowledge. Recently, \citet{hu2024casebased} propose that LLMs perform reasoning by referencing cases similar to the current context, rather than by understanding the underlying reasoning rules.
On the other hand, other researchers emphasize that LLMs go beyond mere memory and mimicry. \citet{yu2023skillmix} evaluated LLMs' capacity to combine fundamental skills to generate novel compositions, revealing that the most advanced LLMs, like GPT-4~\citep{openai2023gpt4}, exhibit this compositional ability, and \citet{arora2023theory} provide a theoretical support for this ability. In an even more optimistic view, \citet{bubeck2023sparks} consider GPT-4 to be ``\textit{an early (yet still incomplete) version of an artificial general intelligence}'' based on their findings that GPT-4 can surpass humans in a wide spectrum of tasks.

Although LLMs have responded to this doubt with their undisputed effectiveness, exactly how their tremendous ability is acquired from their relatively simple training remains an open question. Exploring this mechanism helps us truly understand whether the current LLMs are still ``\textbf{parrots}'', have been \textbf{AGI}, or are developing between these two. In this article, we try to give a possible explanation of this mechanism. Our key observation is that by considering some structural constraints in the generation process, LLMs can acquire reasoning ability through just language modeling, or ``\textit{parroting}''. In other words, \textbf{a parrot can also reason if it learns the structure of language}.

So why do people believe that merely parroting is impossible to master reasoning? From the perspective of learning theory, since each step in the autoregressive generation process involves selecting one of all available tokens and continuing to generate based on this, its possible space increases \textbf{exponentially} with the length of the sequence, making even the largest models, as well as the most extensive training sets, look incredibly tiny in front of the hypothesis space. In this regard, it is reasonable to ask how LLMs can learn to reason various problems in a \textit{generalizable} way. 
This idea has been combined with a controversial argument in linguistics: \textbf{poverty of the stimulus}~\citep{chomsky1987language} by \citet{arora2023theory}. Supporters of this argument believe that in the process of a child learning a language, the examples he or she has been exposed to are insufficient to represent all the features of the language. Therefore, language learning cannot be purely empirical and must involve some understanding of ``universal grammar'', which is the same for LLMs. Understood from this perspective, a stochastic parrot (i.e. an LLM only trained on the next token prediction) seems to be far from understanding a language, let alone having the ability to reason.

In a word, learning a language needs some prominent structural properties. Fortunately, we indeed observe them in many reasoning problems. To solve these problems, the answer words can be \textbf{divided into two parts}: 1) the task-specific \textbf{template} that is relatively fixed in different questions of the same task, and 2) the relatively flexible \textbf{content} that varies with concrete questions. As a flow indicator, the template forms a skeleton to solve the task, guiding the reasoner to split a task into \textit{sub-tasks} and finish the task by filling the template in with the provided content. We provide an example in Figure~\ref{fig:framework}.
Since the template represents a shared solution process for various problems of the same task, the space that models need to learn is no longer exponential with the length of the sequence. Instead, it grows \textbf{linearly} with the number of templates (i.e.,\ tasks) and the sequence length.
As our main contribution, we will point out that this \textbf{widely existing template-content (T-C) structure} in natural language is the key to explain \textbf{why LLMs can generalize well with limited data} (compared to the exponential space) and then explains \textbf{how it is possible for a parrot to reason}. 
We give a more detailed illustration of the (im)possible results and the definition of the t-c structure in Section~\ref{sec:template-content_structure}. 

\begin{figure*}[t]
    \centering
    \vspace{-15pt}
    \includegraphics[width=\textwidth]{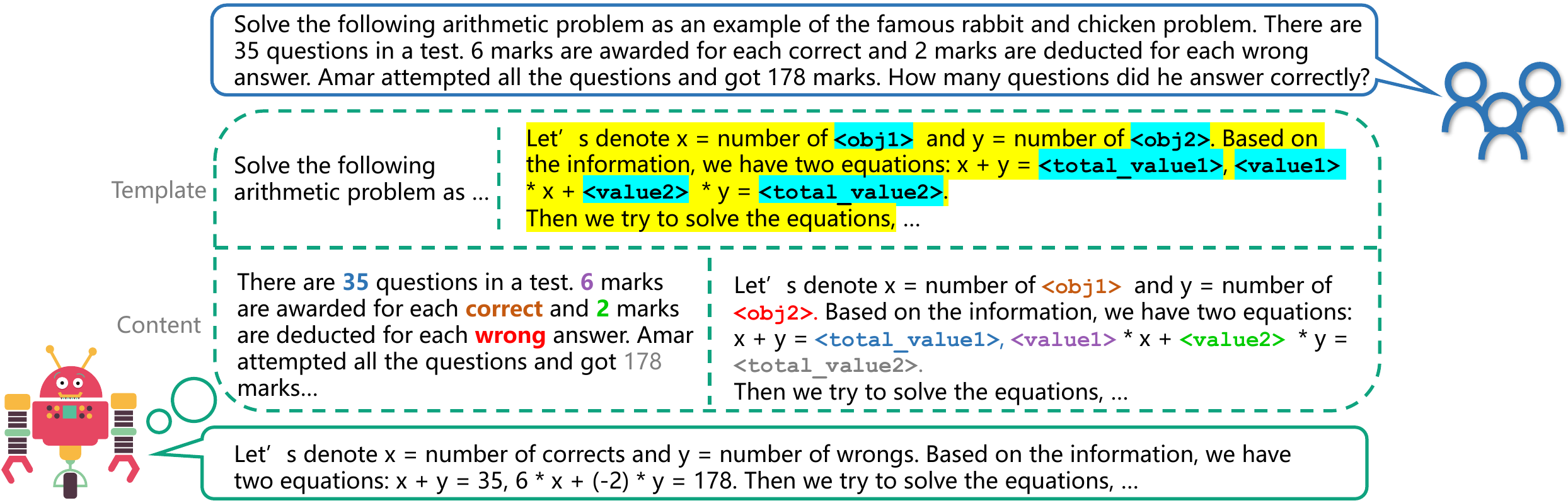}
    \vspace{-15pt}
    \caption{An illustration of the template-content structure. Given prompt and question: 1. The model will generate the \textbf{template} tokens (highlighted as \hl{yellow}) as a flow to solve the task according to the prompt, and some \textbf{content} placeholder (\sethlcolor{lightblue}\hl{blue}) in the template that needs to be filled in, which are displayed in the upper half of the dashed box. 2. The content generation with the guidance of the template could be understood as \textbf{pointing}, shown in the bottom half. Here, corresponding colors shows the pointing process. The combination of these two mechanisms makes reasoning possible.}
    \label{fig:framework}
    \vspace{-15pt}
\end{figure*}

Furthermore, we found that the aforementioned structure primarily elucidates how the model generalizes from one problem to another with the same template but different content, which we term \textbf{within-task generalization}. Nonetheless, for more intricate tasks, this explanation may be inadequate as it disregards the ability of LLMs to combine different sub-tasks. Therefore, we further extend the T-C structure into a hierarchical framework that accommodates varying levels of specificity.
Within this framework, the content of a higher-level template could be further separated into the sub-template and sub-content. When different templates are nested in this manner, these different templates can be combined to generate an overall answer, which involves a composition of different sub-tasks. 
Through this ``\textit{task composition}'' capability, LLMs can achieve generalization on tasks that involve multiple sub-tasks. This ability further reduces the space that LLMs need to remember to the \textbf{logarithm} of the entire task space.
In Section~\ref{sec:hierarchical}, we give many intuitive examples and illustrations about how and why \textbf{the hierarchical T-C structure can explain the task composition ability}.

We hope that our work can provide a new perspective and tool for understanding the reasoning ability of LLMs. Our contributions contain:
    1) We study the question how LLMs with language modeling pretraining can learn to solve reasoning tasks, and we firstly propose the \textbf{template-content} structure as a natural structure of natural language to solve complex tasks, which reduces the learning space from exponential to linear, making it possible for LLMs to achieve within-task generalization. 
    2) By expanding the T-C structure into hierarchical case, we further explain how LLMs can compose the reasoning ability on different sub-tasks to solve an unseen complex task involving multiple sub-tasks. This ``task composition'' ability further reduces the learning space to logarithm.
    3) We conduct experiments to provide evidence that the current LLMs demonstrate the template-content structure and this structure helps models reason.

\section{Background and related work}

\vspace{-3pt}
\textbf{Capacity of LLMs}\hspace{5pt}
Various studies aim to theoretically define the expressive ability of LLMs. Notable examples include investigations into the Turing completeness of Transformers~\citep{turingcomplete, hahn2020theoretical}. \citet{feng2023revealing} discuss the capacity to solve arithmetic tasks and dynamic programming with limited precision and context window. \citet{merrill2023expressive} further prove that Transformers with polynomial context window can solve the P-complete problems. 
Another representative work is the universal approximation theory (UAT)~\citep{yun2020transformers, jiang2023approximation, luoyour}, which demonstrates Transformers' ability to approximate any continuous function. These works give good theoretical guarantees but fail to answer what LLMs can learn in practice.

\vspace{-3pt}
\textbf{Explain the reasoning ability of LLMs}\hspace{5pt}
Numerous studies have explored reasoning ability from the perspective of in-context learning.
\citet{akyurek2023what} illustrated that Transformers can implicitly utilize gradient descent and get closed-form solutions for linear models through given examples. \citet{xie2022an} argue that LLMs acquire document-level concepts during training and utilize in-context examples to recognize the shared concept during inference. Contrary to these approaches, our framework emphasizes template creation and task breakdown rather than in-context learning. Other studies have delved into broader reasoning abilities. \citet{dziri2023faith} propose that LLMs excel on test samples with similar computation graphs to training samples, indicating a form of case-based reasoning. \citet{arora2023theory,yu2023skillmix} suggest that LLMs can amalgamate fundamental skills to cultivate new ones.
\citet{ramesh2023how} further support the idea through their experiments and suggest that intermediate generation is the key to this capacity. \citet{hou2023mechanistic} suggests that LLMs learn an implicit reasoning tree to do multi-step reasoning by experiments.
Very recently, \citet{wang2024understanding} propose that reasoning paths combine concepts into a complete reasoning process, where LLMs learn path aggregation from pretraining corpus to implement reasoning.
Compared to above works, we share a similar understanding that the reasoning ability comes from the combination of basic ``skills'', ``steps'' or ``concepts'', while our T-C structure provides a framework that is both more tangible -- being implementable at the word or token level -- and more rigorously defined.
\vspace{-3pt}

\textbf{Intrinsic dimension}\hspace{5pt}
A related concept is the \textit{intrinsic dimension}, with research by \citet{li2018measuring} showing that data can be effectively learned within a significantly lower-dimensional space than traditionally assumed. This concept, indicative of task complexity, aligns with 
\citet{aghajanyan-etal-2021-intrinsic}'s findings that LLMs exhibit a very low intrinsic dimension, possibly explaining the effectiveness in finetuning. Following this line, our T-C structure describes that the reasoning data inherently exists in a much lower dimension than perceived.

\vspace{-3pt}
\textbf{Template and content}\hspace{5pt}
\citet{ford2018importance} classify words into ``templates'' and ``contents'' based on grammar or frequency and proposes a ``two-pass'' generation. They first generate ``template'' words and placeholders for ``content'' and then replace placeholders, which slightly enhances language modeling. Unlike them, our T-C structure can be defined in a formal way. Additionally, we emphasize the current LLMs (as ``one-pass'' generation) can learn and also benefit from the T-C structure. \citet{madaan2022text} also investigates the different roles of words in the CoT and classify words as \textit{symbol}, \textit{pattern} and \textit{text}. Compared to their empirical divisions exhibiting variations across different tasks, our division maintains consistency thereby yielding generalizable conclusions.

\section{Template-content structure}
\label{sec:template-content_structure}
\subsection{Does language modeling leads to a mindless parrot?}
\label{ssec:impossible}
The most prevalent LLMs follow an autoregressive generative paradigm to solve the problem. Given a prompt and a question sequence, the model generates one next token according to the preceding tokens and then adds the generated token to the input sequence. The iteration is repeated until a final token is generated and all the generated tokens are connected to form the answer. We can formally describe the process as the following definition:
\begin{definition}[Prompt-leading autoregressive models for answer generation]
    For a task and the corresponding prompt sequence $\bm{p}$, a question $\bm{q}$ and a partial answer $\bm{a}_{1:t}$\footnote{We use $\bm a_{1:t}$ to denote the first to $t$-th token sub-sequence or the empty sequence if $t=0$.}, all of which belong to $\mathcal{T}^*$ (the power set of the token space $\mathcal{T}$), a model $\mathcal{M}$ generates the answer autoregressively (until generating the end of text token):
        $\mathcal{M}(\bm{p},\bm{q},\bm{a}_{1:t})= a_{t+1}\in\mathcal{T}.$
    \label{def:language_modeling}
\end{definition}
To training $\mathcal{M}$ as a neural network, it is expected to fit the conditional distribution $p(a_{t+1}\left| \bm p, \bm q, \bm{a}_{1:t} \right.)$. 
\textit{Language modeling} pretraining strategy believes that models can learn the conditional distribution during this pretraining stage. In this stage, the loss is calculated between each predicted distribution and the groundtruth one-hot distribution~\citep{radford2018improving} at each token:
\vspace{-8pt}
\begin{equation}
    L=\sum_{t=1}^{T} \ell\left(\mathcal{M}(\bm p, \bm q, \bm{a}_{1:(t-1)}), \hat{p}(a_t|\bm p, \bm q, \bm{a}_{1:(t-1)})\right),
    \vspace{-8pt}
\end{equation}
where $(p,q,a)$ is the training sequence\footnote{We define the loss only on answer sequence $a$ to align with the sft training technique, which is generally-used in instruction tuning and is believed crucial for reasoning~\citep{ouyang2022training}. Calculating the loss on the entire sequence has no impact on our claim.}, $\mathcal{M}(\bm p, \bm q, \bm{a}_{1:(t-1)})$ is the output distribution of the LLM and the $\hat{p}$ is the one-hot distribution where the element corresponding to the groundtruth next token as 1 and others as 0, which serves as an approximation of the groundtruth distribution $p$. The loss function $\ell$ is commonly chosen as cross-entropy. Because the next token distribution is the only supervision when training, it seems that LLMs mimic the language in corpus like a parrot.

As we mentioned in Section~\ref{sec:intro}, without any constraint on the sequence space, the number of distributions that should be learned will be exponentially increased as $|\mathcal{T}|^{t}$ and rapidly exceeds what even the largest models with the largest data sets can learn. For example, GPT-3 takes $|\mathcal{T}|\approx 10^4$ and the total number of the distributions achieves $10^{20}$ with a sequence length $T=5$, surpassing its $\sim10^{11}$ parameters and $\sim10^{11}$ training tokens.
It seems that \textbf{learning is an impossible task} in this setting. We give a more formal explanation through VC-dimension theory in Appendix~\ref{app:impossible}. Shortly, because of the polynomial upper bound of VC-dimension in the total parameters and computation of neural networks~\citep{KARPINSKI1997169}, models cannot express all the exponential number of output possibilities. 
So, a question needs explanation: how can an autoregressive model pre-trained through language modeling learn various complex tasks?

We must admit that \textbf{there should be an intrinsic structure} in the data distribution of natural language, so that most conditional distributions that need to be learned and generated are restricted to this structure, and it greatly reduces the burden that the model needs to learn and make the learning is possible. In the next section, we will propose a structure that splits the language sequence into two parts and illustrate that this structure is natural for language sequences used in reasoning tasks.

\subsection{Template-content structure}
\label{ssec:tc_structure}
We find that these two properties are intuitive and hold for most reasoning tasks:
\begin{proposition}[Template, informal]
To solve a type of similar reasoning problems, many tokens of answer sequence is almost certain. These tokens form a relatively fixed thinking structure or skeleton, and they are shared for these problems - we call such a type of problems a \textbf{task} and these relatively fixed tokens as \textbf{templates}. Prompts, which often provide task information and hints, are also categorized as templates.
\vspace{-5pt}
\end{proposition}
In Figure~\ref{fig:framework}, we illustrate an example of an ancient Chinese math problem known as the "rabbit and chicken" problem. These problems can be solved using a system of linear equations with two variables. The fixed procedure, highlighted in yellow, involves extracting information to form the equations and solving them for the answer. For instance, when given the tokens ``\dots we have two'', the expected next token is ``equations'', rather than words grammatically inappropriate like ``hello''or unrelated like ``targets''. 
\begin{proposition}[content, informal]
    Other tokens exhibit more flexibility and vary based on the specific problem within a given task type. They are used to distinguish the differences between each problem in the same type of task. \textbf{Replacing them has minimal impact on the generation of the template.} We call these tokens as \textbf{content}. At the same time, we classify the problem as content.
    \vspace{-5pt}
\end{proposition}
In Figure~\ref{fig:framework}, we've highlighted the content in blue, representing specific problem information like numerical values (e.g., \textbf{35} questions, \textbf{6} marks) and objects within the problem. These contents can be replaced and do not affect the intrinsic task ``chicken-and-rabbit problem''. The \textbf{formal definition} of the template and content can be found in Appendix~\ref{app:definition_of_tc}, which could be summarized as a binary token classification with a \textbf{uni-directed} dependency.

The T-C structure can describe the patterns of natural language for complex reasoning tasks, because most of them are \textit{highly structured} and \textit{require a skeleton or decomposition} of tasks such as CoT to complete. Think of the two-step solving scheme: first sketching a basic outline (like a draft) and then filling in specifics. For example, when we solve math problems, we first decide the steps to take and then address each using given details, as depicted in Figure~\ref{fig:framework}. This can be seen as having a general ``\textit{template}'' and specific ``\textit{contents}''\footnote{We admit that the structure does not hold for any task, such as some purely natural language tasks, which do not appear to be obviously structured - such as translation, chatting, etc. Our main focus is on \textbf{complex reasoning problems}.}.
In our experiments (Section~\ref{sec:experiments}) we will show that some most common LLMs, especially those with stronger capacity, clearly \textbf{exhibit two different behaviors} towards different tokens, demonstrating that our modeling here is reasonable.

\subsection{T-C structure makes learning possible}
\label{ssec:t-c_make_possbile}
With the rationality, the next important question is that: \textbf{whether and why learning becomes possible with the T-C structure by reducing the possible space}.
Here we give two aspect illustration. For the \textbf{template} generation, because it is fixed when the preceding template tokens are given, it will not exponentially increase the possible space of conditional distributions. For example, with the fact that the next token should be `\textit{equations}' after `\textit{we have two}', there is no need to learn the distributions $p(a|\dots\textit{we have two xx})$ except for ``\textit{xx}'' is ``\textit{equations}''. Additionally, content words have no impact on template word generation. So there is no need to learn different distributions for a template token with different contents. So at each position, the number of distribution needed to learn is linear to the number of the possible templates (as well as the number of tasks).

As for content words, we further observed that their generation is also simple. They roughly satisfy a ``pointing'' generation pattern. With a suitable template, the capacity of the content filling is \textit{template-dependent pointing}. The generated template leaves blanks to be filled in with specific information, and also provides descriptions or ``roles'' for these blanks.
In Figure~\ref{fig:framework}, with a template like ``x = number of \texttt{<obj1>}'', the left job for the content part is to find the corresponding object, property and value in the problem and copy it (maybe with sightly modification like declension or plural). It is not surprising that the language models have the pointing ability, because it is also the basic ability to understand the semantics and finish many classic NLP tasks such as named entity recognition \citep{chiu2016named, li2020unified} and translation. Formally, we have the following two propositions:

In a word, with T-C structure, the possible space of the distributions needed to learn in pretraining stage to implement reasoning has been considerably reduced, maintaining an order of magnitude approximately linear with the number of tasks times the length of sequence, which makes learning possible.
Here, we want to further explain what the term ``learning'' means. We describe it as a ``\textbf{within-task generalization}''. With the structure-content structure, the procedure to solve a task has been expressed in the template and the template is fixed for the task. It means if a model learns the T-C structure and also learn the template for one task, it has captured the steps to solve the task. Templates can be learned from a few training samples in the huge corpus, but combined with the content pointing ability, all the problems within the task can be solved using the template, which we called ``\textbf{within-task generalization}''.

\begin{proposition}[Transformers can learn the T-C structure]
\label{prop:exist}
    There exists Transformers can learn the T-C structure, which means the generation of template tokens is only based on the preceding template tokens. We denote them as \textbf{T-C Transformers}.
\end{proposition}
\begin{proposition}[The T-C Transformers can achieve the \textbf{within-task generalization}]
    A T-C transformer which has learned the template from a training sample can solve any problem within the task, by continuously writing following the prompt and the question to generate an answer sequence.
\end{proposition}
More detailed definitions, conditions, formal theorems and proofs of these propositions are included in the Appendix~\ref{app:tc-possible}.

\section{Hierarchical template-content structure}
\label{sec:hierarchical}
Above, we demonstrate that the T-C Transformers can achieve the within-task generalization. In order to further enhance the applicability of our explanation, in this section, the capability of within-task generalization can be achieved through nesting to combine different sub-tasks, so that it can cover some very complex scenarios when it may be hard to find a complete template from \textbf{ONE} training sample. Specifically, we extend the template-content structure to the \textit{hierarchical} and \textit{nested} case. This extension entails content corresponding to a template at a given level being decomposable into \textbf{sub-template and sub-content} at the next level. 
\vspace{-5pt}
\renewcommand{\ULdepth}{1.8pt}
\begin{table*}[]
\vspace{-20pt}
\caption{A hierarchical template-content example, where the content is further decomposed into sub-template and sub-content. The different levels of the template are shown as \uline{underline} ($T_1$), \textbf{bold} ($T_2$), and \textit{italic} ($T_3$), with different indents.}
\vspace{-7pt}
\label{tab:hierachy_example}
\centering
\resizebox{0.95\textwidth}{!}{%
\begin{tabular}{lllll}
\toprule
\multicolumn{5}{l}{[Prompt]: \uline{Solve the arithmetic problem step by step.} \textit{Melanie} \textbf{will be} \textit{18} \textbf{years old in} \textit{10} \textbf{years},}\uline{what is} \textbf{the} \textbf{current age} \textbf{of} \textit{Melanie}\uline{?}\\ \midrule
\multicolumn{5}{l}{[Hierarchical T-C structure]: \uline{First, let's identify} \texttt{\textcolor{Green}{\textless{}target value\textgreater{}}}. \uline{According to the problem,} \texttt{\textcolor{violet}{\textless{}information in question\textgreater{}}}.}\\ 
\multicolumn{5}{l}{\uline{This means that} \texttt{\textcolor{blue}{\textless{}write in equation\textgreater{}}}...}\\ \midrule
$\quad$    $\quad$                                     & \multicolumn{4}{l}{\texttt{\textcolor{Green}{\textless{}target value\textgreater{}}}: \texttt{\textcolor{olive}{\textless{}obj\textgreater{}}}\textbf{'s} \texttt{\textcolor{orange}{\textless{}property\textgreater{}}}}                                 \\ 
    & $\quad$  $\quad$                                      & \multicolumn{3}{l}{\texttt{\textcolor{olive}{\textless{}obj\textgreater{}}}: \textit{Melanie},\ \texttt{\textcolor{orange}{\textless{}property\textgreater{}}}: \textit{current age}}                                                                                                      \\ \cmidrule{2-5} 
    & \multicolumn{4}{l}{\texttt{\textcolor{violet}{\textless{}information in question\textgreater{}}}: \textit{Melanie} \textbf{will be} \textit{18} \textbf{years old in} \textit{10} \textbf{years}}                                                                                                  \\ \cmidrule{2-5} 
    & \multicolumn{4}{l}{\texttt{\textcolor{blue}{\textless{}write in equation\textgreater{}}}: \texttt{\textcolor{Thistle}{\textless{}variable\textgreater{}}} \ \textbf{+} \texttt{\textcolor{RubineRed}{\textless{}value1\textgreater{}}} \  \textbf{=} \texttt{\textcolor{Brown}{\textless{}value2\textgreater{}}}}                                                                        \\  
    &                                                & \texttt{\textcolor{Thistle}{\textless{}variable\textgreater{}}}: \textit{age},\quad\texttt{\textcolor{RubineRed}{\textless{}value1\textgreater{}}}: \textit{10},\quad \texttt{\textcolor{Brown}{\textless{}value2\textgreater{}}}: \textit{18}                                                            &                                         &                                         \\ \midrule
    \multicolumn{5}{l}{[Final generation]: \uline{First, let's identify} \textit{Melanie}\textbf{'s} \textit{current age}. \uline{According to the problem},}\\ 
    \multicolumn{5}{l}{\textit{Melanie} \textbf{will be} \textit{18} \textbf{years old in} \textit{10} \textbf{years}. \uline{This means that} \textit{age} \textbf{+} \textit{10} \textbf{=} \textit{18}...}\\ \bottomrule
\end{tabular}%
}
\vspace{-5pt}
\end{table*}
\subsection{Why we need task composition -- some examples}
\vspace{-5pt}
Our T-C structure can explain some tasks well, such as the \textit{chicken-and-rabbit} problem in Figure~\ref{fig:framework} and the \textit{concate-the-last-letter} task\citep{wei2022chain} shown in our experiments section Figure~\ref{fig:classifier} (above). In these two scenes, the content is the objects and numbers in the problem and the words and letters, respectively. But when problems become more complex, the ``\textbf{task}'' with the same template could be too small and too specific, such as more arithmetic problems in SingleEQ~\citep{koncel2015parsing} or GSM8k~\citep{gsm8k} . 
Table~\ref{tab:hierachy_example} provides an example of an arithmetic problem in the SingleEQ dataset where the answer is generated by GPT-4. Considering a non-hierarchical T-C structure, the names ``\textit{Melanie}'' and the numbers $18$, $10$ should be classified as content, but the ``\textit{current age}'', \textit{in \dots years} cannot be replaced because they determine the operation ``+'' so they should be part of the template. It means this template is only helpful to solve problems about \textit{the current age}. This task seems to be \textbf{too specific}, so it is still far from our expectation of generalization.
This phenomenon will be more severe for more complex problems. Considering the following problem from a more complex math dataset GSM8k, 

\textit{``Randy has 60 mango trees on his farm. He also has 5 less than half as many coconut trees as mango trees. How many trees does Randy have in all on his farm?''}

To solve the problem, we should first calculate the ``5 less than half'' as an intermediate result and then add it to the number of mango trees to get the final result. The same template can only help for problems also involving intermediate results calculated from a value and then adding to this value.
This makes our explanation restricted to very detailed generalizations. 
We find a ``\textbf{task composition}'' ability is necessary for this scenario.
For example, assuming there is a problem with a similar first step to the ``mango and coconut trees'' problem --- the process of solving the intermediate variables, but in the next step where the variable is used in another way, we should expect that the model can also learn to how to solve the first step with the help of \textit{this} sample.
\vspace{-5pt}
\subsection{Hierarchical T-C structure leads to task composition}
\vspace{-5pt}
\label{ssec:hierarchical_tc_task_composition}

In the hierarchical T-C structure, the content of the current level could be further decomposed into sub-template and sub-content continuously. For example, the first level content $C_1$ can be decomposed into $T_2$ and $C_2$, and this process continues for $C_2$ into $T_3$ and $C_3$ and so on, until the decomposition is detailed enough. The final structure resemble nested templates labeled as $T_1, \dots, T_n$ (with $T_n$ being content without further divisions). The uni-directed dependence is also required where generating a $T_i$ token is independent to any $T_j$ token where $j>i$. We also give tokens in the prompt and question the same hierarchical T/C classification, dividing them into each level according to different levels of information. 

Back to the example in Table~\ref{tab:hierachy_example}, we demonstrate the nesting relationship between different levels by progressively expanding the content.  This presents a hierarchical structure from coarse- to fine-grained content generation. The prompt sequence ``will be \dots years old in \dots years'' ($T_2$) corresponds to the sub-template ``\dots + \dots = \dots'' ($T_2$), which can be combined with sub-content ``age'', ``10'',``18'' ($T_3$) together to serve as the content \small\texttt{\textcolor{blue}{\textless{}write in equation\textgreater{} }}\normalsize  of $T_1$.
The different levels of templates are corresponding to different levels of specificity, aligning with diverse task segmentation granularity, ranging from the most general to the most specific. In this example, the first level template (\uline{underline}) about a high-level instruction suitable for many arithmetic problem (finding the target value, extracting information in the problem, writing the equation, and returning the answer). The more specific second level template is about converting the age difference into a formula. This separation makes it possible to define a more general task and learn from different training samples. Another example with more levels is in Appendix~\ref{app:hierarchical_generation} Figure~\ref{fig:hierarchy_example}.

\begin{wrapfigure}[15]{r}{0.45\linewidth}
    \vspace{-10pt}
    \includegraphics[width=1\linewidth]{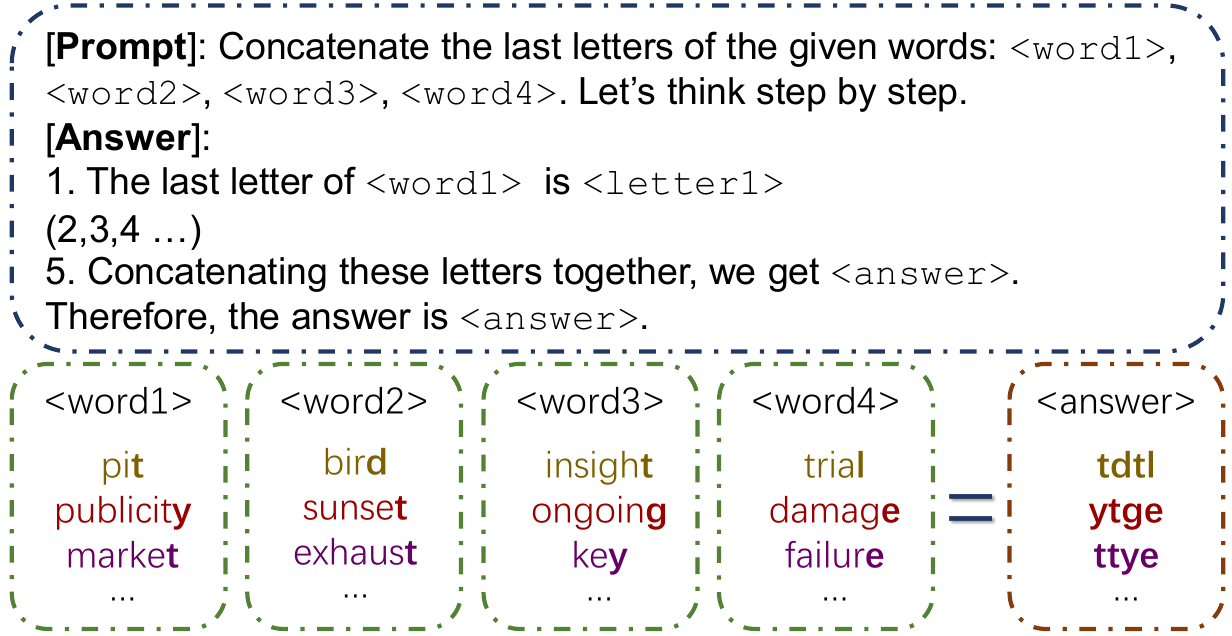}
    \caption{The concatenate-last-letter dataset. The task is to concatenate the last letters of several words together. The template tokens are generated by GPT-4 and fixed, while all the contents (\small\texttt{<word>,<letter>,<answer>}\normalsize) varies.}
    \label{fig:dataset}
    \vspace{-5pt}
\end{wrapfigure}

And in the ``mango and coconut trees'' example from GSM8k, consider the following T-C decomposition. 
The first level template give a general skeleton, which could be like: ``To solve the \small\texttt{\textless{}target\textgreater{}}\normalsize, we should first solve \small\texttt{\textless{}intermediate\textgreater{}}\normalsize. We can solve it by \small\texttt{\textless{}formula\textgreater{}}\normalsize \dots  With the number of \small\texttt{\textless{}intermediate\textgreater{}}\normalsize, we can \dots''  This general template is suitable for a variety of mathematical problems.
Then, several sub-templates focus on solving the intermediate value: the number of coconut trees, which should be filled into the first level content \small\texttt{\textless{}formula\textgreater{}}\normalsize.
To solve the half of mango tress numbers: 60/2=30, two templates like ``half ... /2='' (translating the word ``half'' into an equation ``$\dots/2$''), and ``60/2=30'' (solving the formula) can be helpful. 
The following steps are similar. 

From this process, we find that the answer to this problem can be split into different template fragments, each of which can be independently learned from different training samples and all of which compose a complete answer sequence of the original complex question. This means that \textbf{the number of tasks the model needs to learn is logarithmic in relation to the total number of tasks} with the composition ability.
This property composability ensures that our model can naturally generate this answer by learning these samples separately. We leave formal theorems in Appendix~\ref{app:hierarchical_generation} including the definition of hierarchical T-C structure, the existence of corresponding Transformers, and the theorems that ensure hierarchical answer generation.

\vspace{-10pt}
\section{Experiments}
\label{sec:experiments}

\subsection{Variance of the output}

\label{ssec:variance}
Our T-C structure claims that \textit{the generation should be independent of the content tokens when the next position is template}. 
To test if the real-word models indeed have the T-C structure, we create a T-C dataset.
We use GPT-4~\citep{openai2023gpt4} to obtain answer sequences in the concatenate-last-letter task~\citep{wei2022chain} and labeled the \textit{letters} and \textit{words} as \textit{content} and \textit{the rest} as \textit{template}. After replacing content with other words and letters, we acquire sequences with the same template and the content replacement.
The dataset is illustrated in Figure~\ref{fig:dataset} with details are in Appendix~\ref{app:aligned_dataset}.

\begin{wrapfigure}[17]{r}{0.4\linewidth}
\vspace{-12pt}
\captionof{table}{The correlation between the distinguishablility of T/C tokens and the reasoning performance.}
\vspace{-10pt}
\begin{center}
\resizebox{1\linewidth}{!}{
\begin{tabular}{lcccc}
\toprule
\textbf{Model} & \textbf{AUC-ROC} & \textbf{DMV} & \textbf{CoT} & \textbf{Accuracy} \\ \midrule
Random Guess & 0.5 & 0 & - & - \\
Oracle Model & 1.00 & 1.00 & - & - \\ \hdashline
GPT2-335m      & 0.87             & 0.15         & ×                    & 0.0                        \\
GPT2-774m      & 0.87             & 0.28         & ×                    & 0.0                        \\
GPT2-1.5b      & 0.97             & 0.27         & ×                    & 0.0                       \\
OPT-1.3b       & 0.96             & 0.29         & ×*          & 0.0                        \\
OPT-13b        & 0.99             & 0.42         & ×                    & 0.0                        \\
OPT-30b        & 1.00             & 0.46         & ×                    & 0.0                        \\
Llama2-7b      & 0.99             & 0.72         & $\checkmark$**       & 9.6                       \\
Llama2-13b     & 1.00             & 0.80         & $\checkmark$                    & 16.6                       \\
Llama2-70b     & 1.00             & 0.82         & $\checkmark$                    & 28.4 \\ \bottomrule
\end{tabular}}
\end{center}
\label{tab:variance_and_reason_ability}

* \footnotesize The OPT models cannot generate CoT answers~\citep{zhang2022opt}, so the accuracy is 0.\\
** \footnotesize Llama-2 models have CoT ability to do some complex reasoning~\citep{touvron2023llama}.
\vspace{-20pt}
\end{wrapfigure}

We input these sentences into various open-source models and measure the variance of the output distributions at each position. We calculate the variance for each dimension of output and then add them up. 
The low variance means for different input, the output at this position is relatively fixed while the high variance means the generation will vary depend on the content replacement.
According to our definition, a model has learned the T-C structure should show \textbf{lower} variance for \textbf{template} tokens and \textbf{higher} variance for \textbf{contents}. Here we report the difference of mean variance (DMV) between content and template tokens and AUC-ROC where we use a threshold based on the variance to classify T/C tokens. Here the AUC-ROC measures whether the content and template can be distinguished and the DMV further measures the degree of distinction between the two. The results is shown in Table~\ref{tab:variance_and_reason_ability}. We find most of LLMs shows capacity to distinguish these two types of tokens. For these LLMs with better reasnoning ability, the distinction is significant (with AUC-ROC $\approx 1$ and a large DMV.) 

Another interesting result is that a model's ability to differentiate between T/C seems to correlate with its size and reasoning capabilities.
The results consistent with our theory: \textbf{clear T/C distinction indicates better reasoning ability}. 
To illustrate the correlation between the distinguishablility of T/C tokens and the reasoning performance, we report whether the model has the capacity to generate CoT answer (according to some reports) and the accuracy on the dataset with CoT in Table~\ref{tab:variance_and_reason_ability}. It gives us confidence that (1) the ability to clearly differentiate T/C can serve as a criterion for judging a model's reasoning capability and (2) our framework will likely fit future powerful LLMs.

\subsection{Variance-based autoregressive T/C classifier}
\vspace{-5pt}

To better illustrate the differentiated behavior, we introduce an \textbf{autoregressive} T/C classifier based on the variance. Unlike the experiments in Section~\ref{ssec:variance} which requires manually labeling and write the content list for the whole sentences, we now only do so for the \textbf{prompt}. For any sentence, starting from the word right after the prompt, we \textbf{iteratively} predict the T/C classification word-by-word. For each position, by inputting the \textbf{preceding partial sentences} with several different content replacements, we measure the output variance at the current position. If the variance surpasses a predefined threshold, it is categorized as content, and we record the model's generation as the replacing tokens for it. Otherwise, we classify it as a template and directly add the original word to all perturbed sentences to ensure the same template. This process is repeated until the sentence is fully classified. A pseudo-code is available in Appendix~\ref{app:algorithm}.
\begin{figure}[t]
\vspace{-20pt}
    \centering
    \includegraphics[width=0.45\linewidth]{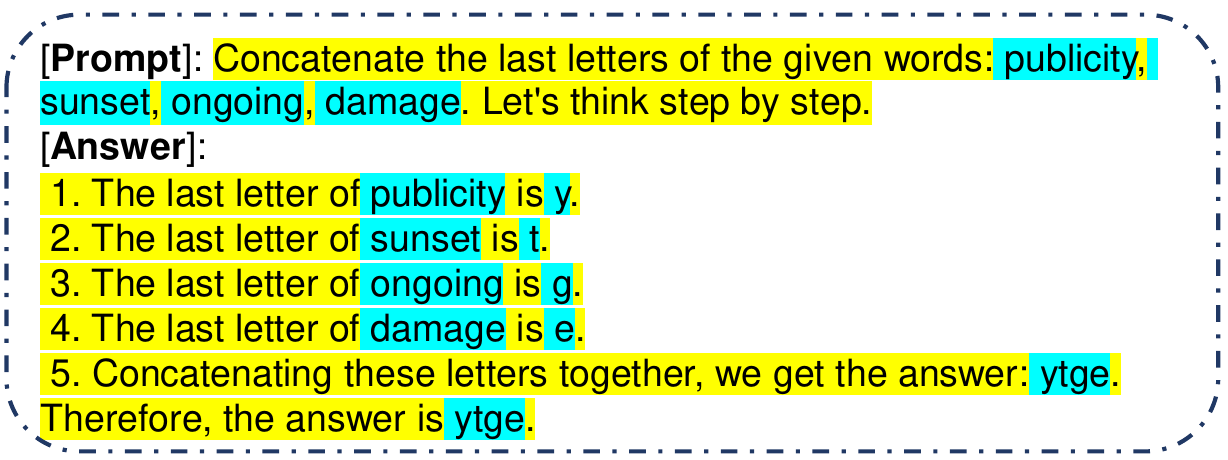}
    \includegraphics[width=0.45\linewidth]{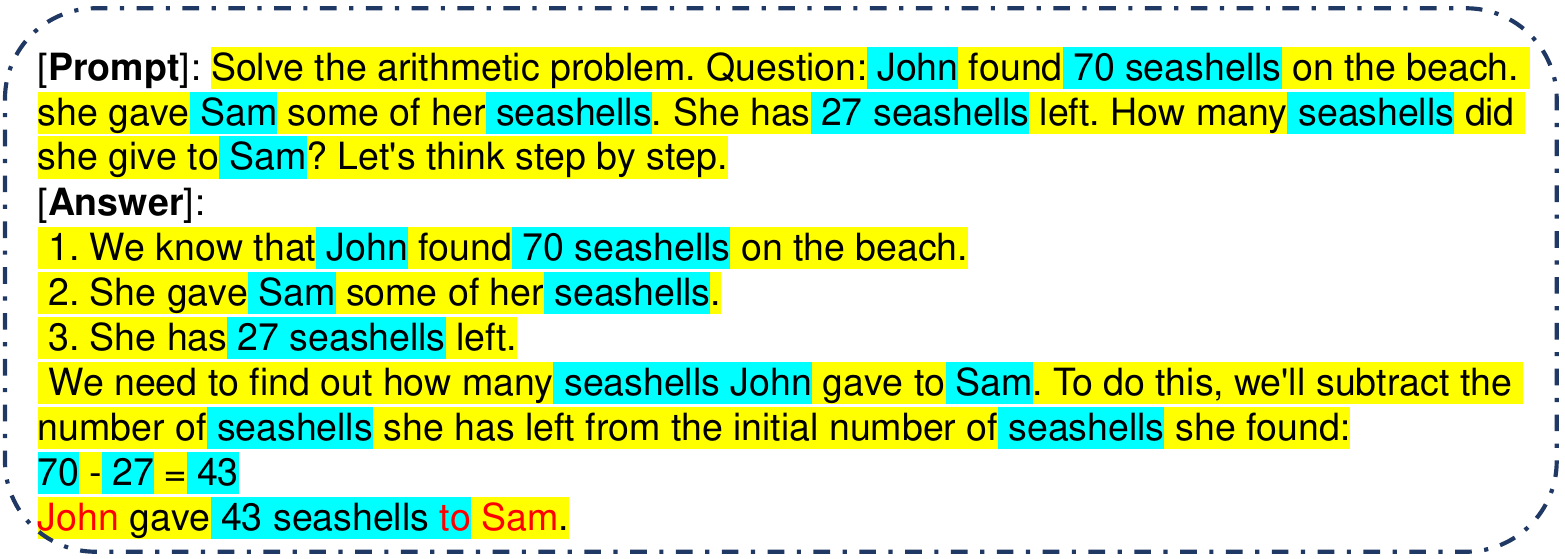} 
    \vspace{-5pt}
    \caption{The T/C classification generated by the autoregressive classifier based on a Llama-2-70b model. Template: \hl{yellow}, content: \sethlcolor{lightblue}\hl{blue}. \textbf{Left}: concatenate-last-letter. \textbf{Right}: SingleEQ. We mark the token whose classification conflicts with human intuition as \textcolor{red}{red}.}
    \label{fig:classifier}
    \vspace{-15pt}
\end{figure}
We test sentences derived from the aforementioned dataset (concatenate-last-letter task)
and a more complex arithmetic dataset SingleEQ~\citep{koncel2015parsing}.
We choose some typical content including the names, objections, and Arabic numbers. The results are shown in Figure~\ref{fig:classifier} and Appendix~\ref{app:more_results_of_tc_classification}. Both of the results are also consistent with our intuition at most positions, which supports our theoretical framework. More details are in Appendix~\ref{app:more_autoregressive_classify}.
\vspace{0pt}
\subsection{T-C structure can improve the performance}
\vspace{0pt}
\begin{wrapfigure}[14]{r}{0.4\linewidth}
    \centering
    \includegraphics[width=0.95\linewidth]{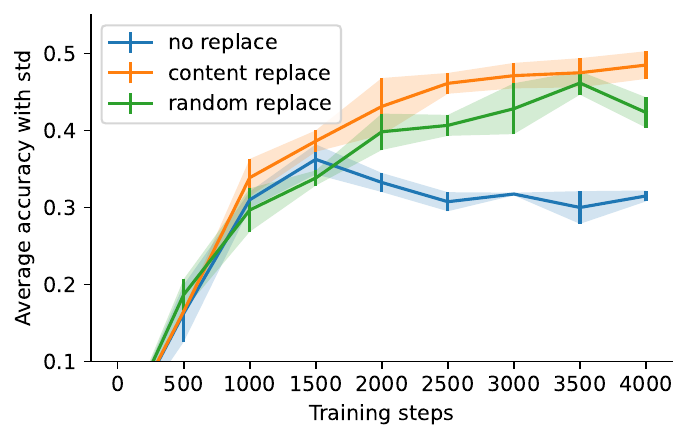}
    \vspace{-5pt}
    \caption{The performance with content replacement, random replacement and no replacement. The shadow is the standard error in 3 times experiments. }
    \label{fig:data_augment}
    \vspace{-10pt}
\end{wrapfigure}
The next question is \textbf{whether the structure can actually help models reason}. Here we try to teach a Llama-2-7b model the T-C structure explicitly and test whether the performance will be improved.
We construct a chicken and rabbit dataset and replace the content during the model training process to help the model learn the TC structure. (Details of dataset is in Appendix~\ref{app:chicken_and_rabbit_dataset}.) The test accuracy is shown in Figure~\ref{fig:data_augment}. In order to avoid the impact of content replacement only as a kind of data enhancement, we also use random replacement as data enhancement, where we randomly replace words with their synonyms~\citep{wei-zou-2019-eda}. We find that our explicit content replacement improves the model's inference performance compared these two baseline. This shows that the model performance has indeed been enhanced through learning the T-C structure. 
\vspace{-5pt}
\section{Conclusions and limitations}
\vspace{-5pt}

This article starts from a longstanding question: whether LLMs are merely mimicking like a parrot, illustrate a constraint on the LLMs generation should be the key to explain its reasoning ability. We introduce the template-content structure, which is an intrinsic structure in language-based reasoning and constrains the learning space for LLMs, enabling LLMs to achieve generalization on inference tasks by learning task-fixed templates from the training data distribution. Additionally, we extend this structure to a higher structural level, demonstrating that LLMs can integrate diverse tasks, further minimizing the learning space necessary for inference tasks.

There are still some limitations of this work.
For our modeling, there is still some gap between our framework and the practical application, such as the alignment, sampling and so on. (Appendix~\ref{app:discussion}). We acknowledge that further refinement of our framework is necessary.
For our experiment, we acknowledge that our experiments rely on limited datasets and models.

\section{Ethics Statement}
Our study focuses on the source of LLMs reasoning ability. Studying the generation mechanism of LLMs can be related to how to promote LLMs to avoid giving false and harmful answers and eliminating model hallucinations. However, the research in this article is not directly related to these social impacts.

\section{Reproducibility}
We attach great importance to the reproducibility of our work. As a work that focuses on theoretical analysis, we give the assumptions and proofs of each theorem as formally and as detailedly as possible; the formal framework is the goal of our work in itself. Although limited by space and coherence, much of the more formal presentation has been added to the Appendix. Regarding the algorithm used in the experiment, we describe in detail all the details including the generation of the data set and the implementation of the algorithm (see Appendix~\ref{app:experiment}). At the same time, we also provide codes in the supplementary materials, through which the experimental results in the article can be directly reproduced.

\newpage
\bibliographystyle{colm2024_conference}
\bibliography{ref.bib}

\newpage
\appendix

\section{The formal definition of T-C structure and formal results}
\label{app:tc-possible}
\subsection{Preliminary}
First let us formally define the causal model, which means the $i$-th output only depends on the preceeding tokens ($\leq i$ tokens).
\begin{definition}[Causal sequence-to-sequence function]
    The sequence-to-sequence function $f:\mathcal{D}\to\mathbb{R}^{n\times d}$ (where $\mathcal{D}\subseteq\mathbb{R}^{n\times d}$) is a causal function, if and only if for any two input $\bm{X}$ and $\bm{X}'$, and \textbf{any} $i\in[n]$, we have:
    $
        f(\bm{X})_{1:i}=f(\bm{X}')_{1:i},\text{ if }\bm{X}_{1:i}=\bm{X}'_{1:i},
    $
    where $\bm{X}_{1:i}$ means the first $i$ rows of $\bm{X}$. Because the first $i$ outputs only depend on the first $i$ inputs, we also denote the causal function as $f:\mathbb{R}^{i\times d}\to\mathbb{R}^{i\times d}$ for any $i\leq n$.
\label{def:causals2s}
\end{definition}

For a Transformers with causal mask like the series of GPT, it is a causal model so we call it a \textit{causal Transformer}.

\subsection{The definition of T-C structure}
\label{app:definition_of_tc}
Here, we define the template and content as a binary classification on the tokens in sequences:
\begin{definition}[Template and content, T/C]
    \label{def:template_and_content}
    Function $\mathcal{F}: \mathbb{N}\times\mathcal{S}\to\{T,C\}$ is called a \textbf{T/C classification function}, where $\mathcal{S}\subseteq\mathcal{T}^*$.\footnote{Here, we define this function only on a subset rather than the entire sequence set because we believe that this T/C classification is meaningful only in ``normal'' natural language rather than in random or garbled text.} The function takes a token sequence $\bm{a}$ and an index $i$ as input and gives the binary classification of the indexed token, denoted as $\mathcal{F}(i;\bm{a})$, abbreviated as $\mathcal{F}(a_i)$. We also use $\mathcal{F}(\bm{a})$ to denote the T/C sequence of the whole \textbf{sequence}.
\end{definition}
 To align with the autoregressive generation and causal model, in the following part, we always assume the T/C classification function (as a sequence-to-sequence function) should be a \textbf{causal} function.
 
We consider the distinct behavior between templates and content as an inherent characteristic of natural language, where the generation of templates is independent of specific contents, while contents relies on templates. 
We require this characteristic by introducing the definition of the \textbf{groundtruth} classification function and its corresponding \textbf{template-content} model (T-C model, which means a model with the save behavior defined in T-C structure).

\begin{definition}[The groundtruth classification and the template-content generation model]
    If a T/C classification function $\mathcal{F}$ and an autoregressive generation model $\mathcal{M}$ satisfy the following requirements, we call the function $\mathcal{F}$ as a \textbf{groundtruth} T/C classification and the model $\mathcal{M}$ as a \textbf{template-content (T-C)} model. The requirements are:
    for any prompt $\bm p$, question $\bm q$ and $\bm q'$, partial answer $\bm a_{1:t}$ and $\bm a'_{1:t}$, (1) if $\mathcal{F}(\bm p, \bm q, \bm a_{1:t}) = \mathcal{F}(\bm p, \bm q', \bm a'_{1:t})$ (T/C alignment) and $a_s=a_s'$ for all $1\leq s \leq t$ such that $\mathcal{F}(a_s)=T$ (the same template), then the T/C classification of the next token is the same:
    \begin{subequations}
        \begin{equation}
            \mathcal{F}\left( \mathcal{M}(\bm p, \bm q, \bm a_{1:t}) \right) = \mathcal{F}\left( \mathcal{M}(\bm p, \bm q', \bm a'_{1:t}) \right),
            \vspace{5pt}
            \label{sequ:next_tc}
        \end{equation}
        and (2) if $\mathcal{F}(\bm p, \bm q, \bm a_{1:t}) = \mathcal{F}(\bm p, \bm q', \bm a'_{1:t})$ (T/C alignment), $a_s=a_s'$ for all $1\leq s \leq t$ such that $\mathcal{F}(a_s)=T$ (the same template), and \textbf{further} $\mathcal{F}\left( \mathcal{M}(\bm p, \bm q, \bm a_{1:t}) \right) = \mathcal{F}\left( \mathcal{M}(\bm p, \bm q', \bm a'_{1:t}) \right)=T$, then
        \begin{equation}
            \mathcal{M}(\bm p, \bm q, \bm a_{1:t}) = \mathcal{M}(\bm p, \bm q', \bm a'_{1:t}).
            \label{sequ:next_token}
        \end{equation}
    \end{subequations}
    \vspace{-10pt}
\label{def:groundtruth_formal}
\end{definition}
Because the definition is somehow complex, here we first illustrate the underlining idea. 
The T-C model and the groundtruth classification defines such an \textbf{ideal} sequence generation schema: every token is classified into template or content, and for two sequences 1) when the preceding T/C sequences align and the template tokens at the corresponding positions are \textbf{the same}, then the classification of the next token to be generated is always \textbf{the same} (which is invariant to different content tokens in the preceding sequences), 2) furthermore, if the next token is a \textit{template}, then the exact token is invariant to different content tokens in the preceding sequences, but only depend on the template tokens.

Additionally, we always consider prompts as templates and questions as contents. For instance, ``the same template'' requires the same prompt while the questions can be different. 
With some approximations, we believe that for natural language, such ideal modeling is appropriate, that is,

\begin{assumption}[Natural language has the template-content structure]
There \textbf{exists} an autoregressive model $\mathcal{M}$ and a T/C classification function $\mathcal{F}$ that satisfy the requirements of the definition of the \textbf{T-C} model and the \textbf{groundtruth} classification.
\end{assumption}
The intuition behind this assumption is that we believe natural language inherently exhibits a hierarchical semantic phenomenon, allowing us to separate the more functional parts from the more specific parts, as mentioned in Section~\ref{ssec:tc_structure}, demonstrated in Figure~\ref{fig:framework} and also talked in \citet{ford2018importance}.
\textbf{Notice that} \textit{here we use $\mathcal{M}$ to describe our ideal T-C-based autoregressive generation schema}, instead of \textit{a specific parametric model or an algorithm}. This ideal model $\mathcal{M}$ describes the underline structure of natural language, especially for language used in the word-based reasoning. 
We give some empirical evidence that the natural language as well as real-world LLMs indeed \textbf{has this T-C structure} in experiment (Section~\ref{sec:experiments}).
We also give a finer version with hierarchical templates in Section~\ref{sec:hierarchical} and also some discussion such as generating a distribution instead of a token in Appendix~\ref{app:discussion}. Here, we focus on this binary setting for simplicity.

\subsection{The existence of the template-content Transformer}
\label{ssec:existence}
Here we will give the formal description and the proof of the Proposition~\ref{prop:exist}: the existence of the T-C Transformers. Specifically, we prove that for a groundtruth classification function $\mathcal{F}$ and its corresponding ideal T-C model $\mathcal{M}_I$, there is a Transformer can simulate the behavior of this model $\mathcal{M}_I$.
\begin{proposition}
    For any groundtruth classification function $\mathcal{F}$ and a T-C model $\mathcal{M}_I$, there exists a Transformer $\mathcal{M}$ can generate the same output of $\mathcal{M}_I$, which means that for any preceding sequence $\bm p, \bm q, \bm a_{1:t}$, $\mathcal{M}(\bm p, \bm q, \bm a_{1:t})=\mathcal{M}_I(\bm p, \bm q, \bm a_{1:t})$.
\label{prop:exist_formal}
\end{proposition}

\begin{wrapfigure}[16]{r}{0.2\textwidth}
    \vspace{-10pt}
    \centering
    \hspace{-17pt}
    \includegraphics[width=0.15\textwidth]{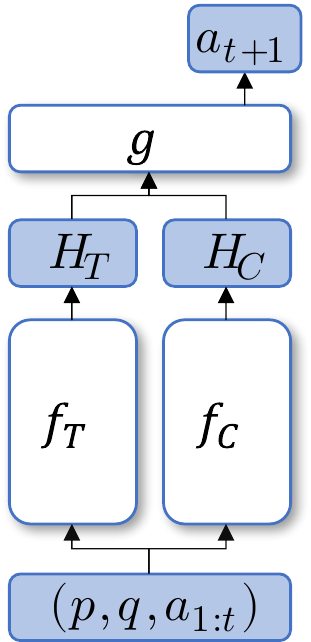}
    \caption{The architecture of the combined Transformer.}
    \label{fig:combine}
\end{wrapfigure}

The framework of our constructive proof is as follows. (1) We can reorganize the generation process of the T-C model $\mathcal{M}_I$: the input tokens are divided into two template tokens and content tokens according to the groundtruth T-C classification $\mathcal{F}$. 
Each group extracts its respective information $H_T$ and $H_C$ by a function $f_T$ and $f_C$, which is then combined by a function $g$ for the final output $a_{t+1}$.
(2) Because of the UAT that we will prove in Appendix~\ref{app:universal},
the information extractor $f_T$, $f_C$ and the final output function $g$ can be all implemented by Transformers. (3) These three Transformers can be combined into a final T-C Transformer. 
Here we provide a diagram showing the architecture of our construction in Figure~\ref{fig:combine}. The detailed constructive proof is as follows:

\begin{proof}
We first rewrite the generation process of the T-C model $\mathcal{M}_I$ as follows:
\begin{equation}
    a_{t+1} = \mathcal{M}_I(\bm p, \bm q, \bm a_{1:t})
    =g_I(f_{T,I}(\bm p, \bm q, \bm a_{1:t}), f_{C,I}(\bm p, \bm q, \bm a_{1:t}))
\end{equation}
where $f_{T,I}$ and $f_{C,I}$ are causal functions that map an input sequence to a sequence with the same length in a certain representation space, and $g_I$ is a causal function to output the next token prediction.
The $f_{T,I}$ is the function to extract the template information independent of any content token. Specifically, it means $f_{T,I}(\bm{p},\bm{q},\bm{a}_{1:t})=f_{T,I}(\bm{p},\bm{q}',\bm{a}'_{1:t})$ if $\mathcal{F}(\bm p, \bm q, \bm{a}_{1:t})=\mathcal{F}(\bm p, \bm q', \bm{a}'_{1:t})$ (T/C alignment), and $a_s=a'_s, \ \forall 1\leq s\leq t:\mathcal{F}(a_s)=T$ (same template). The $f_{C,I}$ is similarly defined as the content representation only depending on the preceding content tokens and being independent of any template token. That is, $f_{C,I}(\bm{p},\bm{q},\bm{a}_{1:t})=f_{C,I}(\bm{p}',\bm{q},\bm{a}'_{1:t})$ if $\mathcal{F}(\bm p, \bm q, \bm{a}_{1:t})=\mathcal{F}(\bm p', \bm q, \bm{a}'_{1:t})$ (T/C alignment), and $a_s=a'_s, \ \forall 1\leq s\leq t:\mathcal{F}(a_s)=C$ (same content). 
To ensure that template token only depends on preceding template tokens, we also have when $\mathcal{F}(a_{t+1})=T$, the value of function $g_{I}$ should only depend on the output of $f_{T,I}$ while being invariant to changes of $f_{C,I}$. 

Considering that the vocabulary is finite, we can set the representation space of $f_{T,I}$ and $f_{C,I}$ as $\mathbb{R}^d$ where for a sufficiently large $d$, the capacity is guaranteed to be adequate. In this situation, the value of $f_{T,I}$ and $f_{C,I}$ are two tensors shaped as $(|\bm p|+|\bm q|+t)\times d$ (denoted as $\bm H_{T,I}$ and $\bm H_{C,I}$), and we can concatenate these two tensors in the second dimension as the input of $g_{I}$. Notice that we can require $f_{T,I}$, $f_{C,I}$ and $g_I$ to be causal functions\footnote{To be more precise, we can construct a function $\hat{g}_I$ that generates a sequence $(\bm p_{2:|\bm p|},\bm q, \bm a_{1:t+1})$, where this function is causal. Then, the $g_I$ takes the last element of the output of $\hat{g}$.} because of the autoregressive generation.
According to Theorem~\ref{lemma:universal}, there \textbf{exist} three causal Transformers that can approximate these function $f_{T,I}$, $f_{C,I}$ and $g_I$ arbitrarily well and we denote them as $f_T$, $f_C$ and $g$.

Finally, we need to show that these three Transformers can be combined into one (through concatenation and stacking). By combining the heads in each layer of $f_T$ and $f_C$ together, the output of the combined layer is the concatenation of the output of $f_T$ and $f_C$. We can keep their independence property by ensuring that the feed-forward layer is block diagonal. As in other theoretical analyses, we ignore the layer norm. After $l=\max(\text{num\_layer}(f_T), \text{num\_layer}(f_C))$ layers, the output of the combined Transformer is the concatenation of $\bm{H}_T$ and $\bm{H_C}$. Then it just needs to add a Transformer $g$ after the combined $f_T$ and $f_C$. As a summary, the constructed Transformers is as shown in Figure~\ref{fig:combine}. Namely, $f_T,f_C$ compose the first few layers of the combined Transformer, their output tensors $\bm{H}_{T},\bm{H}_{C}$ are the hidden representations after these layers, and finally $g$ takes these hidden representations to output the next token.

Until now, we have proven that there \textbf{exists} a Transformer that can generate following the template-content structure. We use $(f_T, f_C, g)$ to represent a T-C Transformer constructed as above. 
\end{proof}

\subsection{Within-task generalization}
\label{ssec:sample-based_template_generation}
Here, we will show that the template-content structure can explain how a pre-trained Transformer gains the ability to generate the template from a training sample, which leads to the task-solving capacity. Because we focus on the generalization, we firstly assume the pre-trained model can remember the training samples.
\begin{definition}[Remember]
    \label{assumpt:well-trained}
    A model $\mathcal{M}$ can \textbf{remember} on a training sample $(\bm{p}, \bm{q}, \bm{a})$, which means given the prompt $\bm{p}$ and question $\bm{q}$, the model can perfectly generate the answer $\bm{a}$ autoregressively.
\end{definition}
This definition requires a pre-trained model to ``remember'' and then reproduce a training example. We believe that such a requirement is not challenging for the prevailing LLMs with a huge amount of parameters, as it does not demand any form of generalization.
Until now, we have demonstrated that (1) the template represents the general process of solving complex tasks in Section~\ref{ssec:tc_structure} and shown in Figure~\ref{fig:framework}; (2) given our modeling of the autoregressive generation task using the T-C structure in Definition~\ref{def:groundtruth_formal}, and our construction of the T-C Transformer in Proposition~\ref{prop:exist_formal}, the T-C Transformer exhibits the ideal behavior that enables template generation independent of specific content. Now, coupling the ideal behavior with the assumption that the Transformer has been trained to remember a training sample to providing the template, it is a natural conclusion that this Transformer can generate the template according to training samples. 

Here we describe the template generation as a partial sequence continuation problem and have the following proposition:

\begin{proposition}[Within-task generalization]
    \label{prop:answer_generation_formal}
    With a pretrained T-C Transformer $(f_T,f_C,g)$, a prompt $\bm{p}$, question $\bm{q}$ and a prefix of a potential answer $\bm{a}_{1:t}$ (empty sequence if $t=0$) as input, assuming that there exists a training sample $(\bm{p},\bm{q}',\bm{a}')$ such that (1) $\mathcal{F}(\bm p, \bm q, \bm{a}_{1:t})=\mathcal{F}(\bm p, \bm q', \bm{a}'_{1:t})$, (2) $a_i=a_i'$ for $1\leq i\leq t:\mathcal{F}(a_i)=T$ and (3) the Transformer can remember the sample, then the Transformer given the sequence $(\bm{p},\bm{q},\bm{a}_{1:t})$ can generate the answer $\bm{a}$ whose template tokens keep the \emph{\textbf{same}} as $\bm{a}'$, i.e., $a_j=a_j'$ for $1\leq j \leq |\bm{a}|: \mathcal{F}(a_j)=T$.
\end{proposition}

\begin{proof}
    Given the prefix $\bm a_{1:t}$ and $\bm a_{1:t}'$, we considering the generation of the next token $a_{t+1}$. Because the Transformer can remember the training sample $\bm p,\bm q', \bm a'$, according to the Definition~\ref{assumpt:well-trained}, we have 
    \begin{equation}
        a'_{t+1}=\mathcal{M}(\bm p,\bm q', \bm a'_{1:t}),
    \end{equation}
    We denote $\mathcal{M}(\bm p, \bm q, \bm a_{1:t})$ as $a_{t+1}$. According to the Equation~(\ref{sequ:next_tc}), we have $\mathcal{F}(a_{t+1})=\mathcal{F}(a_{t+1}')$. If the generated token is a content token, it has satisfied the requirement of the proposition, because the proposition only claims the same template tokens. If the generated token is a template token, according to the Equation~(\ref{sequ:next_token}), we have
    \begin{equation}
        a_{t+1}=\mathcal{M}(\bm p,\bm q, \bm a_{1:t}) = \mathcal{M}(\bm p,\bm q', \bm a'_{1:t})=a_{t+1}',
    \end{equation}
    which satisfies the claim. Finally, we need to check whether the conditions of the proposition have been satisfied for the new sequence concatenated with the generated token $a_{t+1}$. If the conditions are still satisfied, we can finish the proof recursively. For the first condition (T/C alignment), we have $\mathcal{F}(a_{t+1})=\mathcal{F}(a_{t+1}')$ and $\mathcal{F}(\bm p, \bm q, \bm a_{1:t})=\mathcal{F}(\bm p, \bm q', \bm a_{1:t}')$. Because the function $\mathcal{F}$ is a causal function and the newly appended token does not affect the preceding value, we have $\mathcal{F}(\bm p, \bm q, \bm a_{1:t+1})=\mathcal{F}(\bm p, \bm q', \bm a_{1:t+1}')$. The second condition (the same template tokens) is obviously satisfied.
\end{proof}

This proposition conceptually demonstrates that, owing to our template-content modeling of natural language and the construction of the corresponding T-C Transformer, the Transformer can generate an answer sequence with the ``correct'' template as long as there is a training sample with the same template (but content may be different), thereby solving complex problems.

The only condition of the proposition is the existence of the training sample.
The template is considered the \textbf{invariant part} for potentially infinite questions sharing the same task. The template space is much \textbf{smaller} than the total answer space, which is why we can assume the existence of a training sample with such a template that facilitates the LLM to invoke the same template when it sees a new question. This might explain the excellent \textbf{generalization} and \textbf{reasoning} abilities of modern LLMs.
There are still some practical details such as our \textit{token-wise alignment}, \textit{prompt format}, \textit{training process}, and \textit{sampling strategy}, and we discuss them in Appendix~\ref{app:discussion}.

\section{The universal approximation theorem}
we will extend the universal approximation theorem (UAT) to the causal Transformer. The UAT proved in \citet{yun2020transformers} claims Transformers can approximate any continuous function (defined on a compact support set) with arbitrary precision.

\begin{theorem}[Universal approximation theorem, UAT]
    For any continuous distribution $\mathcal{P}$ defined on a compact support $\mathcal{D}\subseteq\mathbb{R}^{n\times d}$, any $1\leq p<+\infty$, $\varepsilon>0$, the context window $n$ and the target continuous sequence-to-sequence function $f:\mathcal{D}\to\mathbb{R}^{n\times d}$, there exist a Transformer $g$ with $l$ layers, such that
    \begin{equation}
        P_{\bm{X}\sim\mathcal{P}}\left[\left\| f(\bm{X}) - g(\bm{X}) \right\|_p<\varepsilon\right] \geq 1-\varepsilon.
    \end{equation}
    where the norm is entry-wise $l_p$ norm.\footnote{Here, we have slightly modified the form of the original theorem, such as the type of convergence.}
\label{thm:uat}
\end{theorem}

This theorem underscores the expressive capacity of Transformers and can serve as the foundation for explaining their reasoning mechanisms. 
However, to the best of our knowledge, there is currently no work that has extended this theorem to causal Transformers, which are the prevalent structures used in large-scale models today.
Therefore, we first provide modifications to this theorem and then utilize it as a tool for studying the autoregressive model and template-content structure.

\begin{theorem}[Universal approximation theorem of causal Transformer]
    For any causal sequence-to-sequence function $f$ (see Definition~\ref{def:causals2s}) and the class of causal Transformer, which means all the attention is masked except for the preceding tokens, the convergence in Theorem~\ref{thm:uat} still holds.
\label{lemma:universal}
\end{theorem}

\label{app:universal}
\begin{proof}
    Here, we follow the proof in \citet{yun2020transformers} (Theorem~3 in that paper). Notice that the only difference between the causal Transformer and the original Transformers is the mask in the self-attention. So we only need to modify the parts of the proof that related to self-attention.

    The proof in \citet{yun2020transformers} can be divided into three steps:
    \begin{enumerate}
        \item Any continuous function defined on a compact support can be approximated by a piece-wise constant function on the $\delta$-grid $\mathbb{G}_\delta=\{0,\delta,\dots,1-\delta\}^{n\times d}$
        \item Any piece-wise constant function can be approximated by a \textit{modified} Transformers. Here, \textit{modified} means that the softmax function in the self-attention is replaced by a hardmax function and the activation function can be any piece-wise linear function with at most three pieces.
        \item The modified Transformers can be approximated by the original Transformers.
    \end{enumerate}
    To modify the proof to the causal setting, the first step can be applied directly, as it is not contingent upon the specific structure of the Transformers. Similarly, the third step, which is proofed by approximating the hardmax through the softmax as the temperature approaches infinity, can be also applied directly.
    The only difference is the second step, the key part of the proof. 
    The basic idea of the second step is several Transformer layers can be used to learn each input vector with its position and context to a unique representation like a \textit{hash} function. Then, with the distinct representation and the universal approximating ability of feed-forward networks, expressive power can be achieved by the following feed-forward functions. 
    Specifically, the part of the proof consists of three-step construction:
    \begin{enumerate}
        \item \textbf{Discretization}: A series of feed-forward layers in the modified Transformers network can quantize the continuous input $\bm{X}\in\mathbb{R}^{n\times d}$ into an element $\bm{L}$ on the extended grid $\mathbb{G}^+_\delta=\{-\delta^{-nd},0,\delta,\dots,1-\delta\}^{n\times d}$. This step is to prepare the unique representation for the inputs.
        \item \textbf{Unique Representation}: a series of self-attention layers can learn the \textit{unique} representation $q(\bm{L}_i;\bm{L})$, where $\bm{L}_{i}$ is the input vector and $\bm{L}$ is the context. The representation is the same only if the input vector and the context are both the same.
        \item \textbf{Value Mapping}: a series of feed-forward layers can map the unique representation to the desired output value.
    \end{enumerate}
    The discretization and value mapping only involve the feed-forward layers and the proof can be applied directly. So we only need to modify the unique representation part. Specifically, modify the global context $\bm{L}$ to causal context $\bm{L}_{1:i}$ for the input $\bm{L}_i$.

    As the claim in Appendix~C in \citep{yun2020transformers}, the proof with position embedding only needs Category~1 in Appendix B.5.1. So we only need to modify this part to fit the causal setting. Here we claim:
    \begin{enumerate}
        \item With two additional dimensions in the hidden states to learn a position embedding, one self-attention layer with two (hardmax) heads can achieve the mapping from the \textit{column id}\footnote{The column id is a unique representation of the input vector. See Appendix~B.5 of the original paper.} $l_k^{(t)}$ into the difference between the current position and the last position $\delta^{-2d}(l_k^{(t)}-l_{k-1}^{(t)})$ 
        \item With $t$ stacks of such layers, the output at $t$-th position is the bijection mapping from the causal context $\bm{L}_{1:t}$.
    \end{enumerate}
    The first claim: for the hardmax attention $\Attn_h(\bm{X})=\sigma_H(\bm{X}\bm{W}_Q(\bm{X}\bm{W}_K)^T)(\bm{W}_V\bm{X})$. 
    
    First, we can use additional two dimensions in the hidden states to store the position embedding $(\cos(t\theta_n), \sin(t\theta_n))$ where $\theta_n=\frac{2\pi}{n}$. With the residual connection between each block, we just need to ensure the output of the attention blocks and feed-forward blocks in these dimensions are all zero so the position encoding will not change through different layers. We denote the extended input as $\bm{L}^+\in\mathbb{G}_\delta^+\times\mathbb{R}^{n\times 2}$.  

    Then let $\bm{W}_Q\in\mathbb{R}^{(d+2)\times2}=(\bm{0}, \bm{R}(-\theta_n))$ where $\bm{R}$ is the rotation matrix in the 2-dimension plain and $\bm{W}_K\in\mathbb{R}^{(d+2)\times2}=(\bm{0}, \bm{I}_2)$. So the $\bm{q}_t =(\cos((t-1)\theta_n),\sin((t-1)\theta_n))$ and $\bm{k}_t = (\cos(t\theta_n), \sin(t\theta_n))$ so that the hardmax at position $t$ always return the index $t-1$ (with additionally defining $\bm{v}_0=\bm{v}_1$). As for $\bm{W}_v$, we just use the construction in the original proof, which means $\bm{W}_v\in\mathbb{R}^{(d+2)\times 1}=(1,\delta^{-1},\dots,\delta^{-d+1},0,0)$. This head returns the \textit{column index} $l_{t-1}$ at the position $t-1$.
    Let another head returns $l_t$ and the $W_h=\delta^{-2d}(1,-1)$, so that after one self-attention layer, the value $\bm{v}_{t}$ at position $t$ is $l_t+\delta^{-2d}(l_t-l_{t-1})$.
    Then we repeat the layer $n$ times, we can easily prove that the value $\bm{v}_t$ at position $t$ is 
    \begin{equation}
    l_t^{(n)}=\sum_{i=0}^n \delta^{-2id}\sum_{k=0}^i\left(\binom{i}{k}(-1)^{k}l_{t-k}\right),
    \end{equation}
    where we use the convention that $l_t=l_1$ if $t\leq0$.
    Now let us show why the value $l_t^{(n)}$ is the bijection mapping from the causal context $\bm{L}_{1:t}$. Note that $|\sum_{k=0}^i(\binom{i}{k}(-1)^{k}l_{t-k})|\leq2^i(\delta^{-d+1}-\delta)\leq(\delta^{-d}-1)$ if we set $\delta\leq1/2^n$. So if we have $l_t^{(n)}=l_t^{(n)'}$, denote $c_i=\sum_{k=0}^i\left(\binom{i}{k}(-1)^{k}l_{t-k}\right)$, then we must have
    \begin{equation}
        \sum_{i=0}^n \delta^{-2id}(c_i-c_i')=0,\quad\text{where }-(\delta^{-d}-1)\leq c_i \leq (\delta^{-d}-1),
        \label{equ:sum_is_zero}
    \end{equation}
    and
    \begin{equation}
        |c_i-c_i'| \geq \delta\text{ or }c_i=c_i'.
    \end{equation}
    because the each $l_i$ differ at least $\delta$ after the discretization and with the special construction of $\bm{W}_v$.

    If $c_n\not= c_n'$, we have
    \begin{equation}
    \begin{aligned}
    &|\sum_{i=0}^{n-1} \delta^{-2id}(c_i-c_i')| \leq \sum_{i=0}^{n-1}2\delta^{-2id}(\delta^{-d}-1) \\ 
    \leq& 2(\delta^{-2nd}-1)/(\delta^{-d}+1)<\delta^{-2nd+1}\leq \delta^{-2nd}|c_i-c_i'|,
    \end{aligned}
    \end{equation}
    but at the same time, 
    \begin{equation}
        |\sum_{i=0}^{n-1} \delta^{-2id}(c_i-c_i')| = |\delta^{-2nd}(c_i-c_i')|=\delta^{-2nd}|c_i-c_i'|,
    \end{equation}
    because of the sum in Equation~(\ref{equ:sum_is_zero}) is zero, there is a conflict. So we must have $c_n=c_n'$. And similarly, we can recursively prove each $c_i$ and $c_i'$ are equal. 
    Note that $c_0=c_0'$ implies $l_t=l_t'$, and $c_1=c_1'$ implies $(l_t-l_{t-1})=(l_t'-l_{t-1}')$ so that $l_{t-1}=l_{t-1}'$. We can recursively prove that $l_i=l_i'$ for any $i=1,\dots,t$. Combining with the bijection of $\bm{L}_{i}$ to $l_t$ (proofed in the original paper), it means the mapping from $\bm{L}_{1:t}$ to $l^{(n)}_t$ is a bijection. Because our results are still bounded, the properties 6.3 and 6.4 can be satisfied with slight modification. This finishes the proof.
\end{proof}

\section{Hierarchical generation}
\label{app:hierarchical_generation}
This section is about the hierarchical T-C structure where we mainly intuitively introduce some definitions needed for the extension, some problems faced, and its motivations in more detail than in main paper. Readers are expected to read this section first before going into more detailed definitions. The formal definitions will be included in Appendix~\ref{app:proof}.

\subsection{Hierarchical T-C structure}
We first generalize the definition of T-C structure (Definition~\ref{def:template_and_content}) into hierarchical setting. 
\begin{definition}[Hierarchical template and content]
    \label{def:hierachical_template_and_content}
    Function $\mathcal{F}: \mathbb{N}\times\mathcal{S}\to\{T_1,T_2,\dots,T_n\}$ is called a \textbf{hierarchical T/C classification function}, where $\mathcal{S}\subseteq\mathcal{T}^*$. The function takes a token sequence $\bm{a}$ and an index $i$ as input and gives the classification of the indexed token, denoted as $\mathcal{F}(i;\bm{a})$, abbreviated as $\mathcal{F}(a_i)$. We also use $\mathcal{F}(\bm{a})$ to denote the T/C sequence of the whole \textbf{sequence}. For each $j\in\{1,2,\dots,n\}$, we denote $T_{< j}$ (or $T_{\leq j}$) as the set $\{T_1, T_2, \dots, T_{j-1}\}$ (or $\{T_1, T_2, \dots, T_j\}$) and $T_{> j}$ as $\{T_{J+1}, T_{j+2}, \dots, T_n\}$ similarly.
\end{definition}
Similar with the non-hierarchical setting, we also assume the hierarchical T/C classification functions are causal functions. We have explain the intuition of the hierarchical T-C structure in Section~\ref{ssec:hierarchical_tc_task_composition} as the uni-directed dependency, where the $j$-th level template could be seen as the content of $T_{< j}$ and the template of $T_{>j}$. Therefore, we require that the generation of $j$-th token depends on only $T_{\leq j}$ tokens. This behavior is similar to the Definition~\ref{def:groundtruth_formal} in non-hierarchical setting, where $C$ depends on $T$ uni-directedly. The formal generalization of the definition could be found in Appendix~\ref{app:assume_hierarchy_independent}.

\subsection{How can we combine different samples into one sequence?}
\label{app:conditions}
As we shown by some examples in Section~\ref{ssec:hierarchical_tc_task_composition}, generalizing the template-content structure to the hierarchical version can enable modeling complex tasks that may involve multiple levels of sub-tasks. However, the answer space also grows exponentially. Looking for a single training sample containing all levels of template guidance may not be realistic for the amount of data required, even for an Internet-scale training set. Fortunately, in this section, we show that it is feasible to combine different-level templates from different training samples. That is, we can learn a combinatorially complex hierarchical template from different training samples, each providing only a certain template, leading to also exponentially increased combinatorial power in the answer generation, which explains the generalization ability of our model.

To formally incorporate the ``combination of different samples'' into our T-C structure,
we need to describe under what conditions these samples can be combined together to form a new sequence.
We first define \textit{label}. In Figure~\ref{fig:framework}, we have already used some blue symbols such as \small\texttt{<obj1>}\normalsize, \small\texttt{<obj2>}\normalsize and \small\texttt{<value1>}\normalsize to denote the \textbf{labels} for the content tokens, which then transforms into concrete tokens such as ``\small\texttt{corrects}\normalsize'', ``\small\texttt{wrongs}\normalsize'' and ``\small\texttt{35}\normalsize''.
We can think of a label as \textbf{sufficient} and \textbf{necessary} information from the template to generate the corresponding content, which means that any modification to the template without altering the label will not influence the generation of the content. For example, ``\small\texttt{based on the formula <equ>}\normalsize'' and ``\small\texttt{according to the equation <equ>}\normalsize'' are two template-content structures with different templates (``based on the formula'' vs. ``according to the equation'') but the same label (``<equ>''). This label will generate exactly the same content (concrete equations) with the same preceding content information (the same arithmetic problem) for these two different templates.
The formal definition is shown in Appendix~\ref{app:formal_definition_of_labal}.
If $\bm a'$ is in the label set of $\bm a$, we say the two sequences $\bm a$ and $\bm a'$ have \textit{label consistency}. Intuitively, it means we can merge the template part of $a$ and the content part of $a'$ together to make a new sentence. And if these two sentences can be generated by an (ideal) T-C model separately, this combined sentence can be also generated by the model.

The concept of \textit{label consistency} can be extended to $n$ samples and used to explain how the combination of $n$ samples can yield the hierarchical templates. 
Specifically, we have $n$ samples $\bm a_1, \dots, \bm a_n$ with the aligned $n$-level T/C classification and want to merge them into one sequence by taking the $k$-th level tokens from the $k$-th sentences. When we have merged $T_{< k}$ levels from their respective samples, we can combine the $k$-th sample if the label of the $k$-th sample at the $k$-th level matches the combined sequence. We say these $n$ samples have \textit{label consistency} if, for any $1\leq k\leq n$, the label of $k$-th sample at the $k$-th level matches the combined sequence. In this situation, we denote the combined sequence as $\hat{\bm a}$. 
Similar to the case of two sentences, this property ensures that as long as each sentence can be generated by a T-C model, the combined sentence can also be generated.
The formal definition is shown in Appendix~\ref{app:formal_definition_of_labal}.

\subsection{The generalization power of the hierarchical template-content modeling}
\label{app:hierarchical_generation_subsection}
Above, the ``\textit{label}'' and ``\textit{label consistency}'' describe the conditions that several samples could be composed together into a new answer sequence.
To achieve the combinatorial generation ability, another issue that needs to be addressed is \textbf{content generation}.
Tokens can simultaneously serve as the content for lower-level tokens and the template for higher-level ones. The generation of tokens could either follow a content-like \textit{pointing} approach or a template-like \textit{continuing writing} strategy.
For example in Figure~\ref{fig:hierarchy_example}, the token \texttt{54} and \texttt{28} should be produced by pointing whereas the resultant subsection \texttt{26} is likely learned from a training sample containing the same calculation. 
As we mentioned in Section~\ref{ssec:t-c_make_possbile}, we focus on the template generating learned from training samples and assume the models have the content-generating ability. 
For the sake of brevity within our framework, we introduce the concept of \textit{virtual training samples} to consolidate these two abilities. Provided a sequence can be autoregressively generated, meeting the criteria in Definition~\ref{assumpt:well-trained}, we will treat it as a training sample.
The sequence's actual presence in the training sample set, or its status as a virtual training sample capable of being generated because of assumed generalization ability (for example, generating content disparate from actual training samples), is inconsequential. 
This approach allows us to uniformly model answer generation within the sample-based continuous generation ability.

With our discussion of the condition that $n$ samples can be combined to provide corresponding-level templates, i.e., the label consistency, and the assumption about the content generation, we can now proceed to provide a proposition similar to Proposition~\ref{prop:answer_generation_formal} in the hierarchical structure.

\begin{proposition}[hierarchical answer generation, informal]
    Given $k$ samples with label-consistency, which a T-C Transformer can remember them,
    this model can generate the combined answer $\hat{\bm a}$ with the same k-level tokens as the sample $\bm a^{(k)}$.
    \label{prop:hierarchy}
\end{proposition}
The formal description and the proof will be shown in Appendix~\ref{app:proof_of_hierarchy}.
The proposition demonstrates the model's capability to combine information from different training samples $\bm{a}^{(i)}$. This capacity leads to exponentially increased combinatorial power in the answer generation as well as the generalization ability.

\subsection{Sparse dependence between levels}
We still need some additional explanations here to further simplify the dependencies between different levels. Notice that it is not necessary that the $k$-level tokens $T_k$ entirely depend on all the lower-level tokens $T_{\leq k-1}$. It means that though the $T_4$, for example, is the sub-content of $T_3, T_2$ and $T_1$ and dependent on them theoretically, the factual dependence could be more sparse than it. A small number of levels may be sufficient to provide enough information to determine the generation. 

\begin{figure*}
\centering
\begin{subfigure}[b]{0.6\linewidth}
    \centering
    \includegraphics[width=1\linewidth]{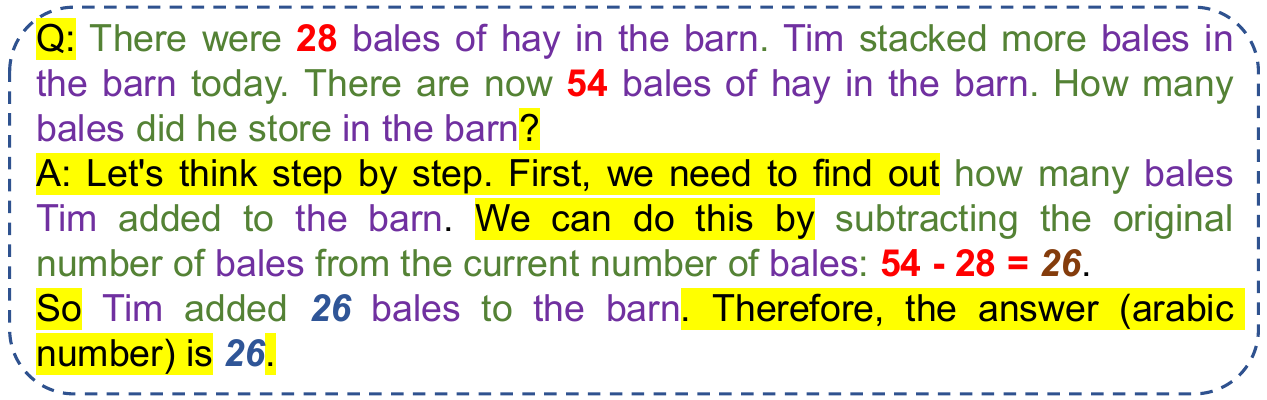}
    \caption{}
\end{subfigure}
\hspace{5pt}
\begin{subfigure}[b]{0.25\linewidth}
\centering
    \includegraphics[width=1\linewidth]{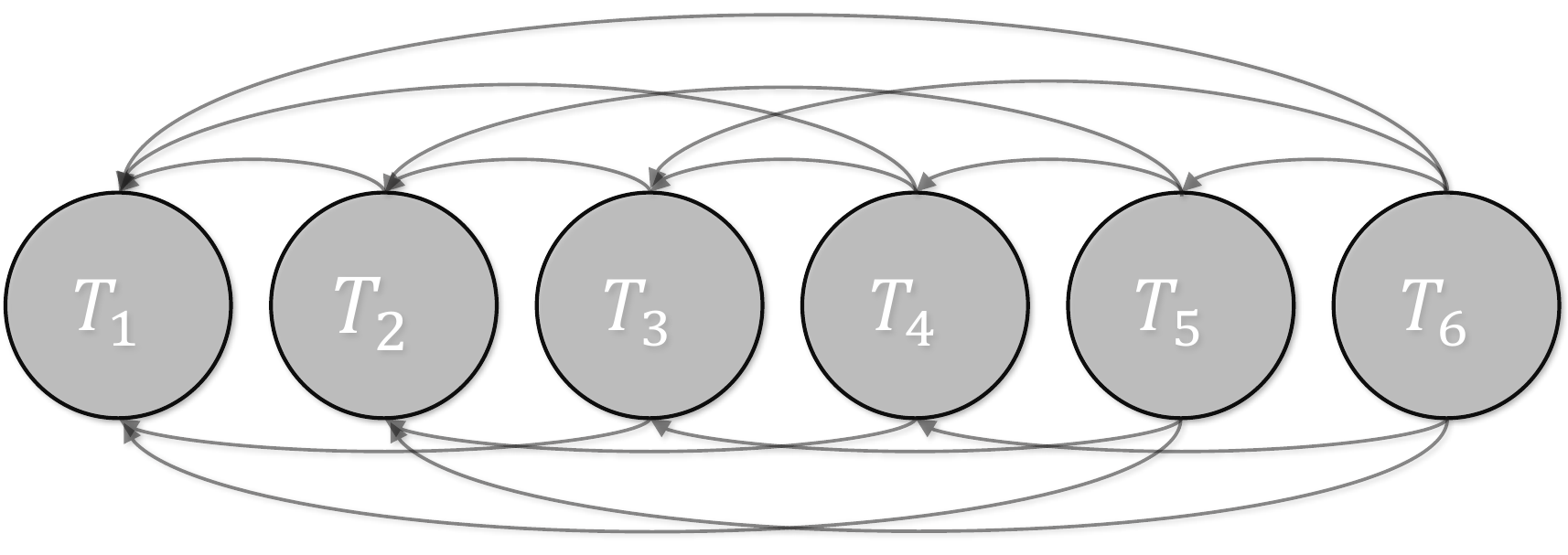}
    \caption{}
    \label{sfig:full_diag}
    \includegraphics[width=1\linewidth]{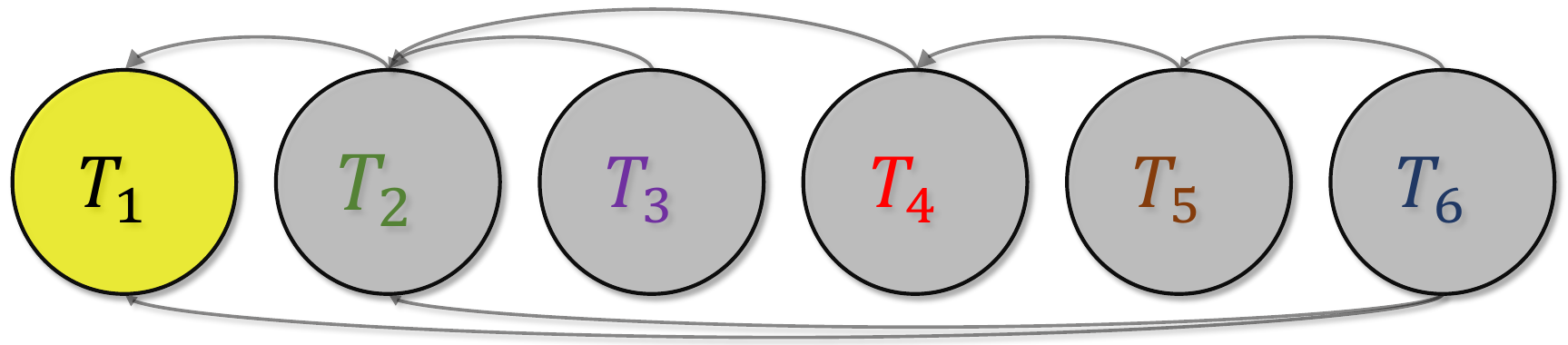}
    \caption{}
    \label{sfig:sparse_diag}
\end{subfigure}
\vspace{0pt}
\caption{Sparsity of the dependency matrix. (a) The sequence is divided into 6 levels: \sethlcolor{newyellow}\hl{$\bm T_1$}, \textcolor{mycolor1}{$\bm T_2$}, \textcolor{mycolor2}{$\bm T_3$}, \textcolor{Red}{$\bm T_4$}, \textcolor{mycolor3}{$\bm T_5$}, \textcolor{mycolor4}{$\bm T_6$}. (b) The dependency structure diagram for the full dependency of 6 levels templates. (c) The practical sparse dependency structure for the sequence in (a), which is much sparser than (b).}
\label{fig:hierarchy_example}
\end{figure*}

An example is in Figure~\ref{fig:hierarchy_example}: solving an arithmetic problem in SingleEQ~\citep{koncel2015parsing} dataset. In this example, the sequence (including the prompt and the generated answer) is segregated into six levels and its dependency is significantly sparser than the full dependency. For instance, the $T_5$ (\textcolor{mycolor3}{brown}) token ``$26$'' only depends on the $T_4$ (\textcolor{Red}{red}) tokens ``$54-28=$'', \textit{regardless of} any information provided by other lower-level templates, such as where the equation appears ($T_1$) or the object bales ($T_3$) and so on.
We call the phenomenon as \textbf{sparse dependency} and give it a formal definition in Appendix~\ref{app:formal_definition_of_support_set}.
Dependency practically tends to be sparser than the full ones.
This phenomenon prevents deeply nested T/C structures from imposing excessive constraints on higher levels (as in the case of full dependency, where the $T_n$ depends on all $n-1$ levels), enabling the model to learn deeper and more complex structures.  

Until now, we expand our T-C template into hierarchical setting and this expansion makes sure that our theory can explain not only
the within-task generalization but also the task composition. With the hierarchical T-C structure, a model can learn, think and answer more like humans. It can learn each part of an answer relatively independently from different training samples and then combine them together to achieve very complex reasoning tasks.

In fact, we also provide: the formal definition of hierarchical T-C structure, the existence of corresponding Transformers, and the corresponding theorems that ensure hierarchical answer generation. Due to space reasons, we put them in appendix for reference.

\section{Formal definition, proposition and proof}
\label{app:proof}
\subsection{Impossible to learn}
\label{app:impossible}
We need the following upper bound of the VC dimension of parameterized class \citep{KARPINSKI1997169}:
\begin{lemma}
Considering the parameterized class
\begin{equation}
    F = \left\{ x\mapsto f(\theta, x):\theta\in\mathbb{R}^d \right\},
\end{equation}
for some $\{\pm 1\}$-valued function $f$. Suppose that, for each input $x\in\mathbb{R}^n$, there is an algorithm that computes $f(\theta, x)$ and this computation takes no more than $t$ operations of the following types:
\begin{enumerate}
    \item the arithmetic operations $+$, $-$, $\times$, $/$ on real numbers,
    \item jumps conditioned on $>$, $\geq$, $<$, $\leq$, $=$ and $\not =$ comparisons of real numbers,
    \item output 0 or 1, and 
    \item the exponential function $\alpha\mapsto e^\alpha$ on real numbers.
\end{enumerate}
Then $\operatorname{VCdim}(F)=O(t^2d^2)$.
\end{lemma}
It is easy to check a Transformers meets these assumptions so that the VC dimension can be also bound as $O(t^2d^2)$. Actually, we only need the upper bound is a polynomial function. For a Transformers with input length $T$, layers $L$ and an upper bound of all the hidden dimension $d'$, the total computation is also a polynomial function of $T,L$ and $d'$. Therefore, the VC dimension of a Transformer with input length $T$ can be upper bound by a polynomial function of $T$. Referring to the definition of VC dimension, this means that Transformer cannot shatters the exponential number of input points. Notice that the VC dimension theory is defined in modeling the problem as a binary classification problem. However, in a binary classification problem the model no longer has enough expressive power, and it is also impossible to learn the continuous-valued vector output - the latter is what we really issue concerned.

\subsection{The formal definition of hierarchical template independence}
\label{app:assume_hierarchy_independent}
Here, we give the generalization of the Definition~\ref{def:groundtruth_formal} in the hierarchical situation. First, we need to generalize the conditions as follows:
\begin{definition}[$k$-th level alignment]
With a $n$ level hierarchical classification function $\mathcal{F}$, two sequences $\bm a$ and $\bm a'$ with same length are called \textit{$k$-th level aligned}, if and only if for any $1\leq t \leq |\bm a|$, at least one of the following conditions is satisfied:
\begin{equation}
    \mathcal{F}(a_t)=\mathcal{F}(a_t')\text{ and }a_t=a_t',
\end{equation}
or
\begin{equation}
    \mathcal{F}(a_t)\in T_{\geq k+1}\text{ and }\mathcal{F}(a_t')\in T_{\geq k+1}.
\end{equation}
\end{definition}
To show that the definition is the generalization about the T/C alignment and fixed template tokens, notice that if we set $n=2$ and $k=1$, the second condition says that both the tokens are content tokens while the first condition says that both are template tokens and the same. In this situation, the condition is the same as the ones that we use in Definition~\ref{def:groundtruth_formal} and Proposition~\ref{prop:answer_generation_formal}, i.e., the T/C classification is aligned and the template tokens are the same.
Then, the generalization of the Definition~\ref{def:groundtruth_formal} is as follows:
\vspace{5pt}
\begin{definition}[The groundtruth classification and the template-content generation model, hierarchical]
\label{def:fixed_tc_classification_hierarchy}
    For a hierarchical T/C classification function $\mathcal{F}$ with $n$ level and an autoregressive generation model $\mathcal{M}$, for any $1\leq k \leq n$, given the \textit{$k$-th level aligned} prefix sequences $\bm a_{1:t-1}$ and $\bm a_{1:t-1}'$, if the generated (next) token is also $k$-th level aligned during the autoregressive generation, we call the function $\mathcal{F}$ as a \textbf{groundtruth} T/C classification and the model $\mathcal{M}$ is a T-C model. It means for all sequences $\bm a$ and $\bm a'$ with the same length, and for all $1\leq t \leq |\bm a|$, \textbf{at least one} of the following equations should be satisfied:
    \begin{equation}
    \begin{aligned}
        &\mathcal{F}\left( \mathcal{M}(\bm a_{1:t-1}) \right) = 
        \mathcal{F}\left( \mathcal{M}(\bm a'_{1:t-1}) \right)\text{ and }\\ &\mathcal{M}(\bm a_{1:t-1})=\mathcal{M}(\bm a'_{1:t-1}),
    \end{aligned}
    \end{equation}
    \textbf{or}
    \begin{equation}
        \mathcal{F}\left( \mathcal{M}(\bm a_{1:t-1}) \right)\in T_{\geq k+1},\text{ and }
        \mathcal{F}\left( \mathcal{M}(\bm a'_{1:t-1}) \right)\in T_{\geq k+1},
    \end{equation}
    if the prefix sequences $\bm a_{1:t-1}$ and $\bm a_{1:t-1}'$ are $k$-th level aligned.
\end{definition}

\subsection{The formal definition of dependency matrix}
\label{app:formal_definition_of_support_set}
Formally, for a given sentence $\bm a$, when the next token is $T_{k}$, there is a \textbf{support set} $\bm D_k\in\{0,1\}^k$ such that
\begin{equation*}
\begin{aligned}
    g(f_{T_1}(\bm a), \dots, f_{T_k}(\bm a)) = g(f_{T_1}(\bm a^{(1)}), \dots, f_{T_k}(\bm a^{(k)})),\\ \forall \bm a^{(s)}\in \mathcal{S}^{(s)}, s=1,2,\dots,k,
\end{aligned}
\end{equation*}
where $\mathcal{S}^{(s)}$ is a one-point set $\{\bm a\}$ if the $s$-th element of $\bm D_k$ (denoted as $d_{ks}$) is $1$, otherwise $\mathcal{S}^{(s)}$ is the full space of the sequence with the same length.
In this situation, we refer to the $k$-th level as (conditional) \textbf{independent} of the $s$-th levels if $d_{ks}=0$.
The support set must exist because if we set it as $(1,1,\dots,1)$, it says nothing about the function $g$. Then We define the \textbf{dependency matrix} as a lower-triangle matrix where $\bm {D}_{k,1:k}=\bm D_k$. 

\subsection{The formal definition of label}
\label{app:formal_definition_of_labal}
Formally, we define

\begin{definition}[Label]
    For the given model $\mathcal{M}=(f_{T_1},\dots,f_{T_n},g)$, a sequence $\bm{a}_{1:t}$ and its dependency matrix $\bm D$, when the next token to be generated $a_{t+1}$ satisfies $\mathcal{F}(a_{t+1})=T_{k}$, the $k$-th level label of $\bm a_{1:t}$ (denoted as $\mathcal{L}_k(\bm a_{1:t})$) is defined as the set of sequences $\{\bm a_{1:t}'\}$:
    \begin{enumerate}[leftmargin=*]
        \item The T/C classification is aligned, i.e., $\mathcal{F}(\bm a_{1:t}')=\mathcal{F}(\bm a_{1:t})$.
        \item The combined sequence $\hat{\bm a}_{1:t}$ is constructed by replacing all the $\leq (k-1)$-level tokens in $\bm a_{1:t}$ with those in $\bm a'_{1:t}$.
        The replacement does not affect
        \begin{enumerate}[leftmargin=*]
            \item the T/C classification of the next token, which means
            \begin{equation*}
                \mathcal{F}(\mathcal{M}(\hat{ \bm {a}}_{1:t})) = \mathcal{F}(\mathcal{M}(\bm a_{1:t}))
            \end{equation*}
            \label{item:tc_class}
            \vspace{-10pt}
            \item the generation of the next token, which means 
                \begin{equation*}
                \begin{aligned}
                    &g(f_{T_1}(\bm a'_{1:t}),\dots,f_{T_{k-1}}(\bm a'_{1:t}), f_{T_{k}}(\bm a_{1:t}))
                    \\=&g(f_{T_1}(\bm a_{1:t}),\dots, f_{T_{k}}(\bm a_{1:t}));
                \end{aligned}
                \end{equation*}
                \label{item:generation}
                \vspace{-10pt}
            \item the dependency matrix, which means the dependency matrix $\bm D$ of the sequence $\bm a_{1:t}$ is also the dependency matrix of sequence $\hat{\bm{a}}$.
            \label{item:dependence}
        \end{enumerate}
    \end{enumerate}
\end{definition}
\vspace{10pt}
With the statement~\ref{item:tc_class}, \ref{item:generation} in the definition, we can rewrite the generation as
$a_{t+1}=g(f_{T_1}(\mathcal{L}_{k}(\bm a_{1:t})),\dots,f_{T_{k-1}}(\mathcal{L}_{k}(\bm a_{1:t})),$$f_{T_{k}}(\bm a_{1:t}))$, if $\mathcal{F}(a_{t+1})= T_{k}$.
It is worth noting the similarity between this definition and the definition of T-C model (Definition~\ref{def:groundtruth_formal} and \ref{def:fixed_tc_classification_hierarchy} in Appendix). To ensure the same next token generation, the ``having the same label'' condition is a relaxation of the ``having the same $\leq T_{k-1}$ token'' condition, whereby the exact template can be substituted with equivalent sequences.
For the sequence $\bm a'$ in the label set $\mathcal{L}_{k}(\bm a)$, the information from the dependent-level template tokens of $\bm a'$ is equivalent to those of $\bm a$.
It enables the possibility of combining two sequences from different sources during the ``continuation'' generation of the T-C model.
That is, the $k$-level template part from the original sequence $\bm a$ can be combined with the lower-level template part from another sequence $\bm a'$, which does not impact the generation of $T_{k}$ tokens, just like we can combine the template ``\texttt{according to the equation}'' with the content in ``\texttt{based on the formula <equ>}'' to generate an appropriate sentence ``according to the equation <equ>''.

With the statement~\ref{item:dependence}, it further claims the replacement does not introduce additional inter-level dependency. When the support set describes the dependency of the original sequence $\bm a$, the sequence can be replaced solely on these dependent levels, while the arbitrariness of other independent levels can still be maintained. For example, when we replace ``\texttt{based on the formula <equ>}'' with ``\texttt{according to the equation}'', the replacing sequence keeps the independence of further lower-level templates and therefore we can combine them.

We formally define \textit{label consistency} as follows:
\begin{definition}[Label consistency]
\label{def:label_consistency}
    For a set of sequences $\{\bm a^{(k)}\left|\right.k=1,\dots,n\}$ with $n$ levels of template, they have \textit{label consistency} if and only if 
    the following requirements are satisfied.
    \begin{enumerate}
        \item The T/C classification of these sequences is aligned.
    \begin{subequations}
        \begin{equation}
            \mathcal{F}(\bm a^{(i)}) = \mathcal{F}(\bm a^{(j)}),\quad\forall i,j\in \{1,\dots,n\}.
            \label{sequ:template_alignment}
        \end{equation}
        \item We construct the \textbf{combined sequence} $\hat{\bm a}$
        which takes the $T_k$ tokens from the corresponding sequence $\bm a^{(k)}$, i.e., $ \hat a_t =  a^{(k)}_t$ if $\mathcal{F}(a^{(i)}_t)=T_k,\forall i$.
        For any $0\leq t\leq |\hat{\bm{a}}|-1$, and $\mathcal{F}(\hat{a}_{t+1})=T_{k}$,
        \begin{equation}
            \bm \hat{\bm a}_{1:t}\in \mathcal{L}_k(\bm a^{(k)}_{1:t}).
        \label{sequ:compatibility}
        \end{equation}
        \vspace{-15pt}
    \end{subequations}
    \end{enumerate}
\end{definition}
The first requirement (Equation~(\ref{sequ:template_alignment})) serves the purpose of constructing the combined sequence $\hat{\bm a}$. As the key requirement, the second one (Equation~(\ref{sequ:compatibility})) ensures that the $k$-th sample can be merged with the 1st to $(k-1)$-th samples ``\textit{appropriately}'', collectively forming the $k$-th template, by requiring its $k$-th level label should match the combined sequence $\hat{\bm{a}}$. Here, when we say ``appropriately'', intuitively, it means that the combined sentence remains coherent and reasonable and, thus, for an ideal T-C autoregressive model, it can still generate the same $k$-level tokens.
At the same time, this requirement also determines the dependency matrix of the sequence $\hat{\bm{a}}$. Specifically, the $k$-th row of the dependency matrix is the same as the $k$-th row of the dependency matrix for the $k$-th sample.

\subsection{The proof of Proposition~\ref{prop:hierarchy}}
The formal description of the Proposition~\ref{prop:hierarchy}.
\begin{proposition}[Hierarchical answer generation, formal]
    Given a partial answer $\bm{a}_{1:t}$ as the input and an $n$-level hierarchical template-content classification function $\mathcal{F}$, and a T-C Transformer model $(f_{T_1},\dots,f_{T_n},g)$, we assume that there exist a set of training samples $\{\bm{a}^{(k)}|k=1,\dots,n\}$ that has \textit{label consistency} (with denoting the combined sequence as $\hat{\bm a}$), and the partial sequence $\hat{\bm{a}}_{1:t}=\bm{a}_{1:t}$ and the model $\mathcal{M}$ can remember them.
    We have the model can generate the answer $\bm a$ from $\bm{a}_{1:t}$ autoregressively as the same tokens as the combined sequence $\hat{\bm a}$, i.e., with the same $k$-level tokens as the training sample $\bm{a}^{(k)}$ for any $k$ from $1$ to $n$. Here, we require that each prompt of each sample $\bm p^{(k)}$ should be contained in the partial sequence $\bm a_{1:t}$ or be generated as a part of $T_{\leq k-1}$.
\end{proposition}
\vspace{10pt}

\label{app:proof_of_hierarchy}
\begin{proof}
Here, we prove that if the assumption $\hat{\bm a}_{1:t+s}=\bm a_{t+s}$ holds for any the partial answer $a_{t+s}$ where $s$ takes value from $0$ to $l-t-1$, then $\hat{\bm a}_{1:t+s+1}=\bm a_{t+s+1}$ also holds, where we use $a_{t+s}$ to denote the generated partial sequence with length $t+s$. 

Without loss of generality, we assume the next token $\hat{\bm a}_{1:t+s+1}$ is a $T_k$ token. Because the training samples have label consistency, we have the T/C classification of the generating token $a_{1:t+s+1}$ is also a $T_k$ token, i.e.,
\begin{equation}
    \mathcal{F}(a_{t+s+1}) = \mathcal{F}(\hat{a}_{t+s+1})=T_k
\end{equation}
Therefore, the token can be generated by
\begin{equation}
    a_{t+s+1} = g(f_{T_1}(\bm a_{1:t+s}), \dots, f_{T_k}(\bm a_{1:t+s})).
\end{equation}
And we have the assumption $\hat{\bm a}_{1:t+s}=\bm a_{t+s}$, so we have
\begin{equation}
    a_{t+s+1} = g(f_{T_1}(\hat{\bm a}_{1:t+s}), \dots, f_{T_k}(\hat{\bm a}_{1:t+s})).
\end{equation}
Because the label consistency, we have $\hat{\bm a}_{1:t+s} \in \mathcal{L}_k(\bm a^{(k)}_{1:t+s})$ and notice that $f_{T_k}(\bm a^{(k)}_{1:t+s})) = f_{T_k}(\hat{\bm a}_{1:t+s}))$ because the combined sequence has the same $T_k$ tokens as the $k$-th sample $\bm a^{(k)}$, so we have
\begin{equation}
     g(f_{T_1}(\hat{\bm a}_{1:t+s}), \dots, f_{T_k}(\hat{\bm a}_{1:t+s})) = g(f_{T_1}(\bm a^{(k)}_{1:t+s}), \dots, f_{T_k}(\bm a^{(k)}_{1:t+s})).
\end{equation}
At the same time, according to the assumption of well-training and the prompt, we have that the model can generate the token $a^{(k)}_{t+s+1}$ given the sequence $\bm a^{(k)}_{1:t+s}$, so we have
\begin{equation}
    a_{t+s+1} = a^{(k)}_{t+s+1}.
\end{equation}
According to the definition of the combined sequence $\hat{\bm{a}}$, $\hat{a}_{t+s+1}$ takes value from $a^{(k)}_{t+s+1}$. So we finally prove that $a_{t+s+1} = \hat{a}_{t+s+1}$. 
\end{proof}

\section{Discussion of the template-content framework}
\label{app:discussion}
\begin{remark}[Practical training process]
    In our assumption, training samples should be in the format as the (prompt, question, answer) triplets while the Internet corpus may not be primarily presented in such a format. We point out that this alignment could be achieved in the crucial finetuning process i.e. RLHF or similar finetuning phase. In this phase, the model can realign previously learned corpus into the triplet format. 
\end{remark}

\begin{remark}[format of the prompt]
    There are two primary prompt types: zero-shot instruction and few-shot exemplars and our template-content structure is applicable to both. Zero-shot prompts aid the model in recalling the templates learned during training, while few-shot prompts can also provide explicit templates. By utilizing these few-shot exemplars, models can generate templates based on the exemplars and also leverage the knowledge from similar templates, while the content information should not be directly used in answer generation. This observation also explains why zero-shot learning is more challenging than few-shot learning, as zero-shot learning necessitates the model to independently generate templates. This concept aligns with the findings in \citet{min2022rethinking}, demonstrating that the primary performance improvement of few-shot prompt stems from describing the space of questions and answers rather than direct Q\&A mapping.
\end{remark}

\begin{remark}[Alignment of position]
    \label{remark:position_alignment}
    We claim semantic alignment is a more realistic setting for our T-C framework. However, employing semantic alignment introduces several challenges, such as describing position correspondence and considering the variations in position encoding.
    Nevertheless, we find it reasonable to embrace token-wise alignment, as we believe that a well-trained model can automatically bridge the gap between these two settings during training. By disregarding position offsets irrelevant to semantics, token-wise alignment can be achieved for semantically aligned samples through the introduction of blank characters.
    
    From a training point of view, the model may be able to quickly learn to what extent the position offset is irrelevant to semantics, such as an additional space, so it does not affect any representation and generation. With this ability, if the model fits well on one training sample, it can also fit well on another sample, which is only different in some position offsets. Based on the observation, we can assume the existence of the latter sample (i.e., the sample with the position offsets) and use the token-wise alignment assumption. 
    
    As an additional explanation for this ability, we believe that the semantic and position information are disentangled in the Transformer.
    From a simple test, it is easy to test that the semantic embedding and position embedding are roughly orthogonal for most open-source models and therefore disentangled, which means the model can capture the semantic information and position information separately. Second, some results are also shown that the semantic information and the position information can be learned by different heads. See the analysis in~\citet{voita2019analyzing}. 
\end{remark}
    
\begin{remark}[Output probability and diversity]
    Another simplification in our theoretical framework pertains to the sampling scheme. We assume there is one \textit{standard} token at each position while a more realistic setting involves the model's output being a probability distribution over the vocabulary and the output token is then sampled from the distribution. With random sampling, the model can generate diverse output with the same input.

    Fortunately, our framework readily accommodates this sampling approach with simple modifications. We posit the existence of a distribution over the sequence space, which can be learned from a sufficiently large corpus. Given the prompt, question, and partial sequence, the distribution is projected into a conditional probability, describing the probability of the subsequent tokens. 
    By replacing the individual training sample (or the sample set, if considering the hierarchical template) with this distribution, our framework seamlessly adapts to this setting. Formulating the sampling process into our framework will be our future work.
\end{remark}

\section{Experiments details and more results}
\label{app:experiment}

\subsection{Concatenate-last-letter dataset}
\label{app:aligned_dataset}
The template of the concatenate-last-letter dataset is:\\

Concatenate the last letters of the given words: \texttt{\textless{}word1\textgreater{}}, \texttt{\textless{}word2\textgreater{}}, \texttt{\textless{}word3\textgreater{}}, \texttt{\textless{}word4\textgreater{}}.\\
Let's think step by step.\\
1. The last letter of \texttt{\textless{}word1\textgreater{}} is \texttt{\textless{}letter1\textgreater{}}.\\
2. The last letter of \texttt{\textless{}word2\textgreater{}} is \texttt{\textless{}letter2\textgreater{}}.\\
3. The last letter of \texttt{\textless{}word3\textgreater{}} is \texttt{\textless{}letter3\textgreater{}}.\\
4. The last letter of \texttt{\textless{}word4\textgreater{}} is \texttt{\textless{}letter4\textgreater{}}.\\
5. Concatenating these letters together, we get \texttt{\textless{}answer\textgreater{}}.\\
Therefore, the answer is \texttt{\textless{}answer\textgreater{}}.

We produce the dataset by: (1) We collect the top 5,000 most commonly occurring English words from wiktionary\footnote{\href{https://en.wiktionary.org/wiki/Wiktionary:Frequency\_lists/English/Wikipedia\_(2016)}{https://en.wiktionary.org/wiki/Wiktionary:Frequency\_lists/English/Wikipedia\_(2016)}}. (2) Randomly sample words as \texttt{\textless{}word\textgreater{}} and extract the corresponding letters and results.

\subsection{Chicken-and-rabbit dataset}
\label{app:chicken_and_rabbit_dataset}
To construct the dataset, we first search 100 different problems on the Internet and then use GPT-4 to generate 5 answers for each question. There is 500 question-answer pairs. To make content replacement, we also use GPT-4 to help to replace them by giving them the question and the answer, with the prompt: ``Please help me to replace the content such as ``Here is an example of the chicken and rabbit problem and its step-by-step solution. Please replace the number in both the problem and the solution to make a new Q\&A pair. Please ensure all other tokens all the same. Please confirm the answer is still correct. '' Finally, we manually filter some generation with incorrect answer.

\subsection{variance of output}
\label{app:variance_of_output}
The specific variance at each position are shown in Figure~\ref{fig:more_variance}, 
where the bar shows the variance of the output distribution at each position. The variance has been normalized so the maximal value $1$ means the typical content, where the output is different one-hot vetors corresponding to the variable content tokens while the minimal value $0$ means the typical template, where the output distribution is all the same for different content replacement. In the figures, shorter template \textcolor{Green}{green} bars compared to 10 content \textcolor{Cyan}{blue} bars (4 words, 4 letters and 2 answers) indicate significant less variance on the template positions. The results suggests that real-world models behave as the T-C model defines. In other word, our T-C structure can be applied to real-world models.

Here all of Llama-2 models we used are fine-tuned by chat data, i.e., the Llama-2-xxb-chat-hf model proposed by Huggingface. The ROC curve is shown in Figure~\ref{fig:ROC1}. 

We also test on some other answer sequences generated by GPT-4 on the concatenate-last-letter task and follow the same process (labeling the content, replacing words and letters). These datasets differ in the specific template and content list but follows the same generating process. These templates are as follows:

Concatenate the last letters of the given words: \texttt{\textless{}word1\textgreater{}}, \texttt{\textless{}word2\textgreater{}}, \texttt{\textless{}word3\textgreater{}}, \texttt{\textless{}word4\textgreater{}}.\\
Let's think step by step.\\
1. Word: \texttt{\textless{}word1\textgreater{}}, last letter: \texttt{\textless{}letter1\textgreater{}}.\\
2. Word: \texttt{\textless{}word2\textgreater{}}, last letter: \texttt{\textless{}letter2\textgreater{}}.\\
3. Word: \texttt{\textless{}word3\textgreater{}}, last letter: \texttt{\textless{}letter3\textgreater{}}.\\
4. Word: \texttt{\textless{}word4\textgreater{}}, last letter: \texttt{\textless{}letter4\textgreater{}}.\\
Now, let us concatenate the last letters of each word: \texttt{\textless{}letter1\textgreater{}} + \texttt{\textless{}letter2\textgreater{}} + \texttt{\textless{}letter3\textgreater{}} + \texttt{\textless{}letter4\textgreater{}} = \texttt{\textless{}answer\textgreater{}}. Therefore, the concatenated result is \texttt{\textless{}answer\textgreater{}}.

The results are shown in Figure~\ref{fig:more_variance_2} and \ref{fig:ROC2}. The conclusion is the same as we shown in Figure~\ref{fig:more_variance} and in Section~\ref{ssec:variance}.

\begin{figure*}
    \centering
    \begin{subfigure}[b]{0.32\textwidth}
        \includegraphics[width=\textwidth]{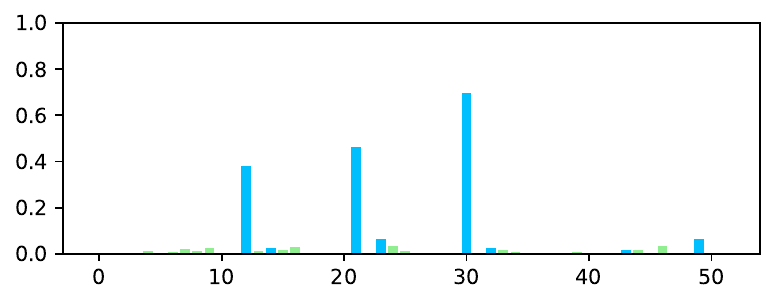}
        \caption{GPT2-medium(335b)}
    \end{subfigure}
    \hspace{-5pt}
    \begin{subfigure}[b]{0.32\textwidth}
        \includegraphics[width=\textwidth]{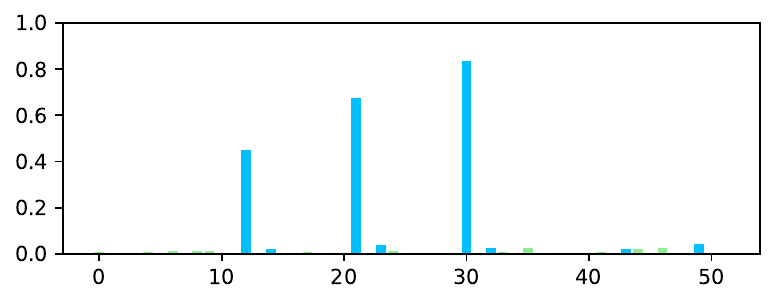}
        \caption{GPT2-large(774m)}
    \end{subfigure}
    \hspace{-5pt}
    \begin{subfigure}[b]{0.32\textwidth}
        \includegraphics[width=\textwidth]{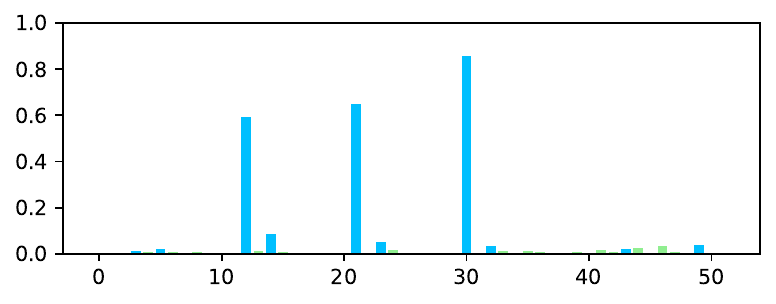}
        \caption{GPT2-xl(1.5b)}
    \end{subfigure}
    \hspace{-5pt}
    \begin{subfigure}[b]{0.32\textwidth}
        \includegraphics[width=\textwidth]{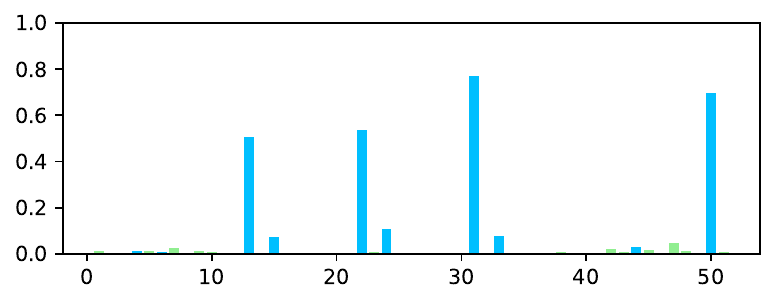}
        \caption{OPT-1.3b}
    \end{subfigure}
    \hspace{-5pt}
    \begin{subfigure}[b]{0.32\textwidth}
        \includegraphics[width=\textwidth]{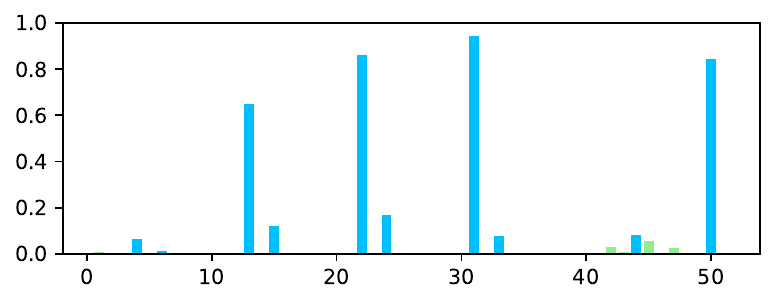}
        \caption{OPT-13b}
    \end{subfigure}
    \hspace{-5pt}
        \begin{subfigure}[b]{0.32\textwidth}
        \includegraphics[width=\textwidth]{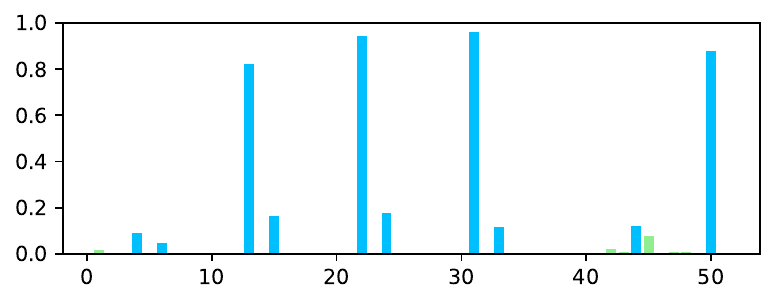}
        \caption{OPT-30b}
    \end{subfigure}
    \hspace{-5pt}
    \begin{subfigure}[b]{0.32\textwidth}
        \includegraphics[width=\textwidth]{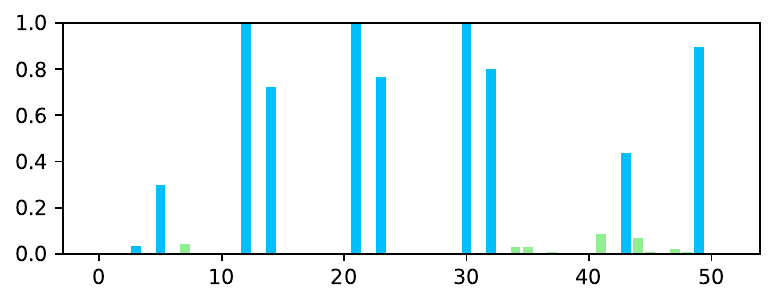}
        \caption{Llama2-7b}
    \end{subfigure}
    \hspace{-5pt}
    \begin{subfigure}[b]{0.32\textwidth}
        \includegraphics[width=\textwidth]{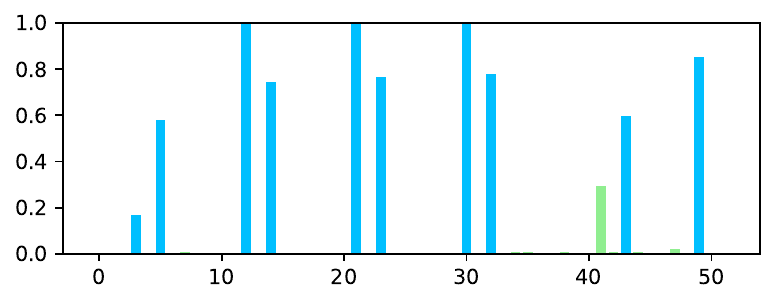}
        \caption{Llama2-13b}
    \end{subfigure}
    \hspace{-5pt}
    \begin{subfigure}[b]{0.32\textwidth}
        \includegraphics[width=\textwidth]{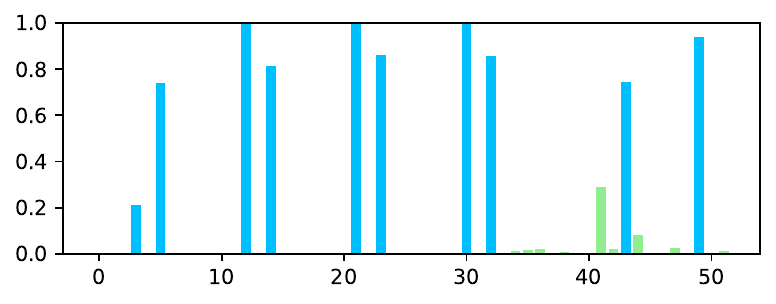}
        \caption{Llama2-70b}
    \end{subfigure}
    \caption{More results of the variance of the output distributions at template and content positions. X-axis: index of the tokens, y-axis: coefficient of variation. Blue bars: content tokens, a total of 10; green bars: template tokens.}
    \label{fig:more_variance}
    \vspace{-12pt}
\end{figure*}

\begin{figure*}[tbhp]
    \centering
    \includegraphics[width=0.7\textwidth]{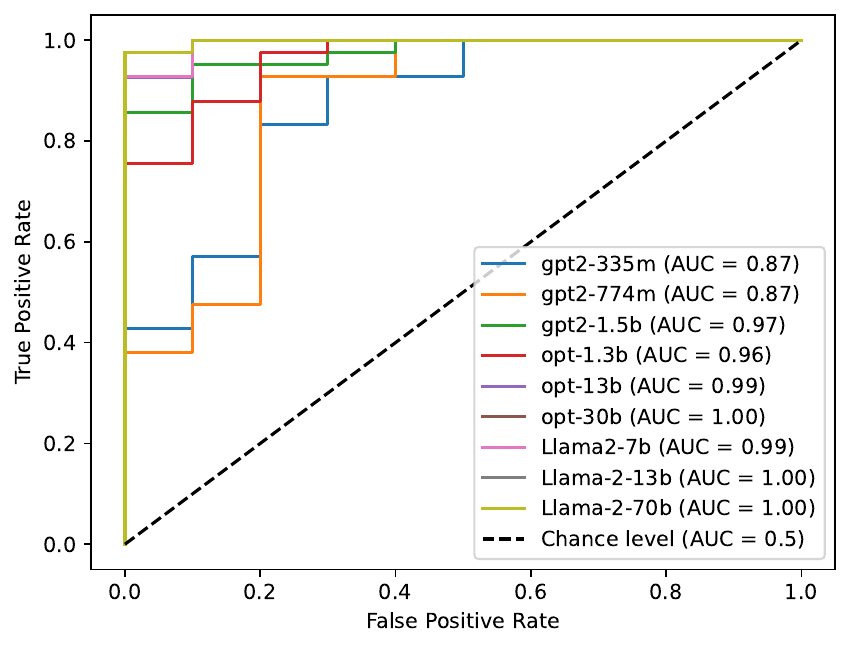}
    \caption{Receiver Operating Characteristic (ROC) of template tokens and content tokens.}
    \label{fig:ROC1}
    \vspace{-12pt}
\end{figure*}

\begin{figure*}
    \centering
    \begin{subfigure}[b]{0.32\textwidth}
        \includegraphics[width=\textwidth]{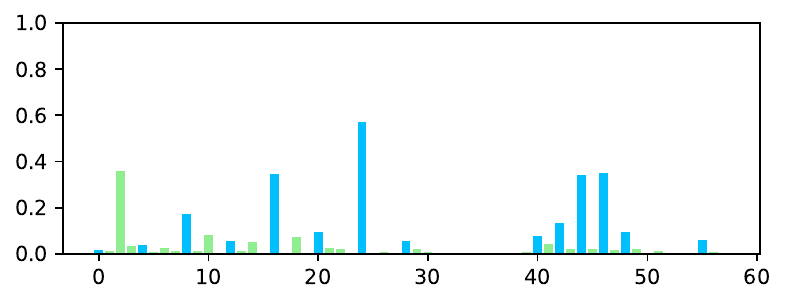}
        \caption{GPT2-medium(335b)}
    \end{subfigure}
    \hspace{-5pt}
    \begin{subfigure}[b]{0.32\textwidth}
        \includegraphics[width=\textwidth]{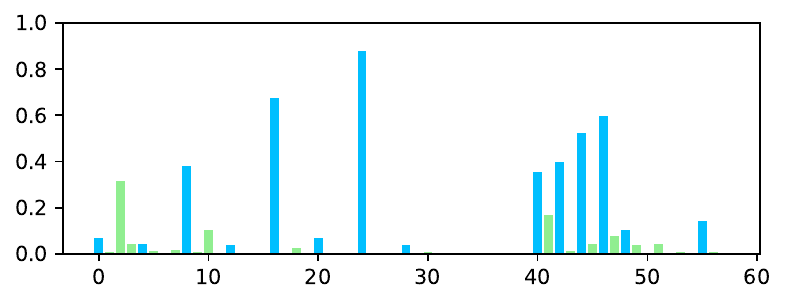}
        \caption{GPT2-large(774m)}
    \end{subfigure}
    \hspace{-5pt}
    \begin{subfigure}[b]{0.32\textwidth}
        \includegraphics[width=\textwidth]{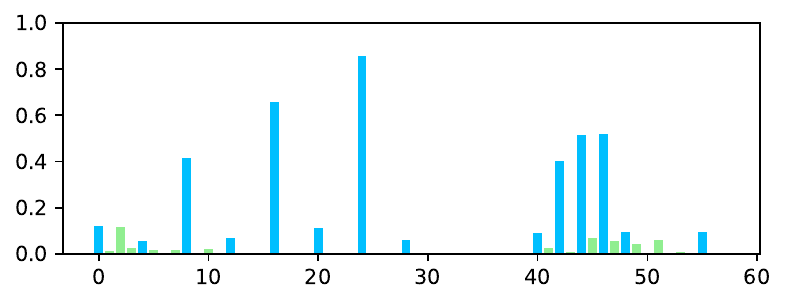}
        \caption{GPT2-xl(1.5b)}
    \end{subfigure}
    \hspace{-5pt}
    \begin{subfigure}[b]{0.32\textwidth}
        \includegraphics[width=\textwidth]{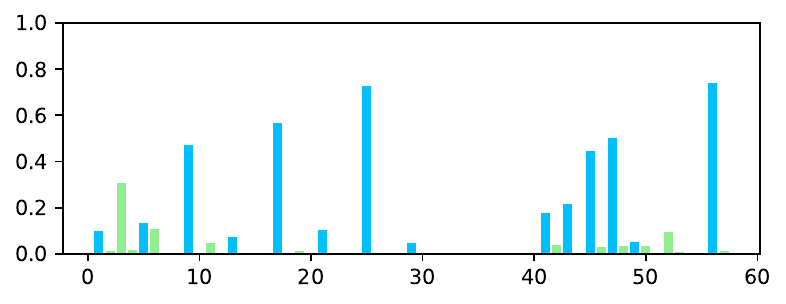}
        \caption{OPT-1.3b}
    \end{subfigure}
    \hspace{-5pt}
    \begin{subfigure}[b]{0.32\textwidth}
        \includegraphics[width=\textwidth]{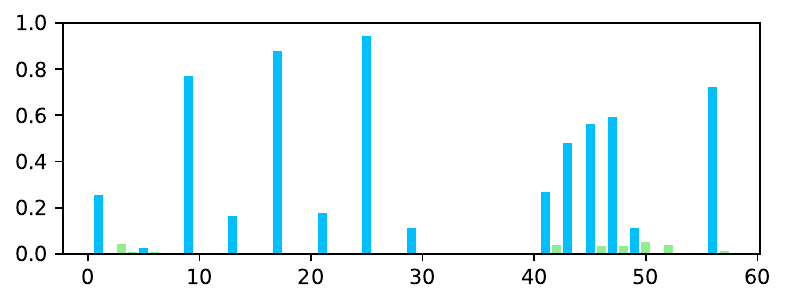}
        \caption{OPT-13b}
    \end{subfigure}
    \hspace{-5pt}
        \begin{subfigure}[b]{0.32\textwidth}
        \includegraphics[width=\textwidth]{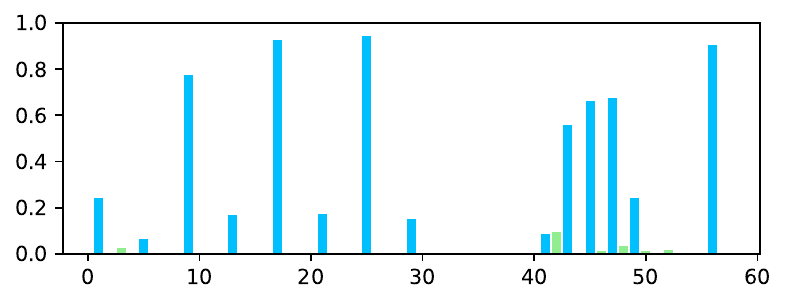}
        \caption{OPT-30b}
    \end{subfigure}
    \hspace{-5pt}
    \begin{subfigure}[b]{0.32\textwidth}
        \includegraphics[width=\textwidth]{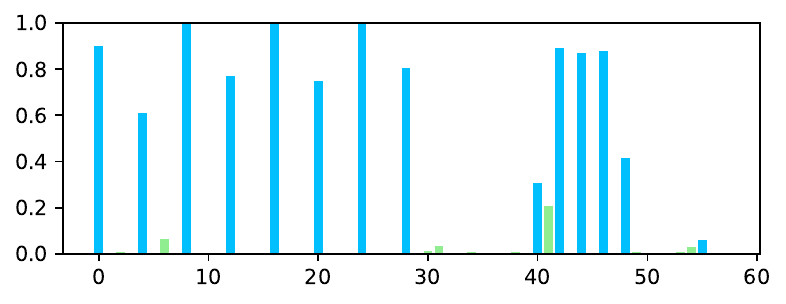}
        \caption{Llama2-7b}
    \end{subfigure}
    \hspace{-5pt}
    \begin{subfigure}[b]{0.32\textwidth}
        \includegraphics[width=\textwidth]{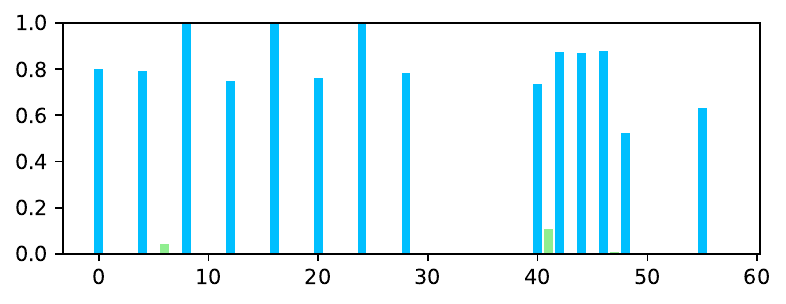}
        \caption{Llama2-13b}
    \end{subfigure}
    \hspace{-5pt}
    \begin{subfigure}[b]{0.32\textwidth}
        \includegraphics[width=\textwidth]{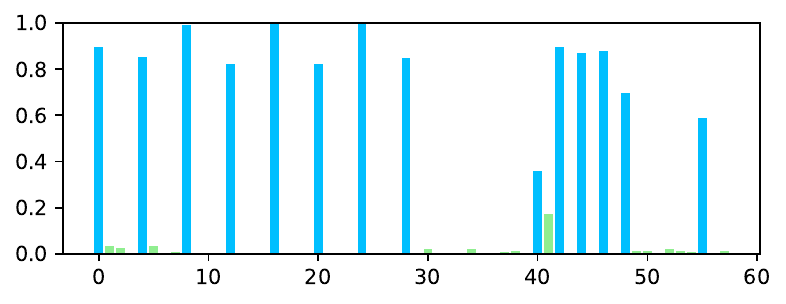}
        \caption{Llama2-70b}
    \end{subfigure}
    \caption{More results of the variance of the output distributions at template and content positions with \textbf{another template}. X-axis: index of the tokens, y-axis: coefficient of variation. Blue bars: content tokens, a total of 14; green bars: template tokens.}
    \label{fig:more_variance_2}
    \vspace{-12pt}
\end{figure*}

\begin{figure*}[tbhp]
    \centering
    \includegraphics[width=0.7\textwidth]{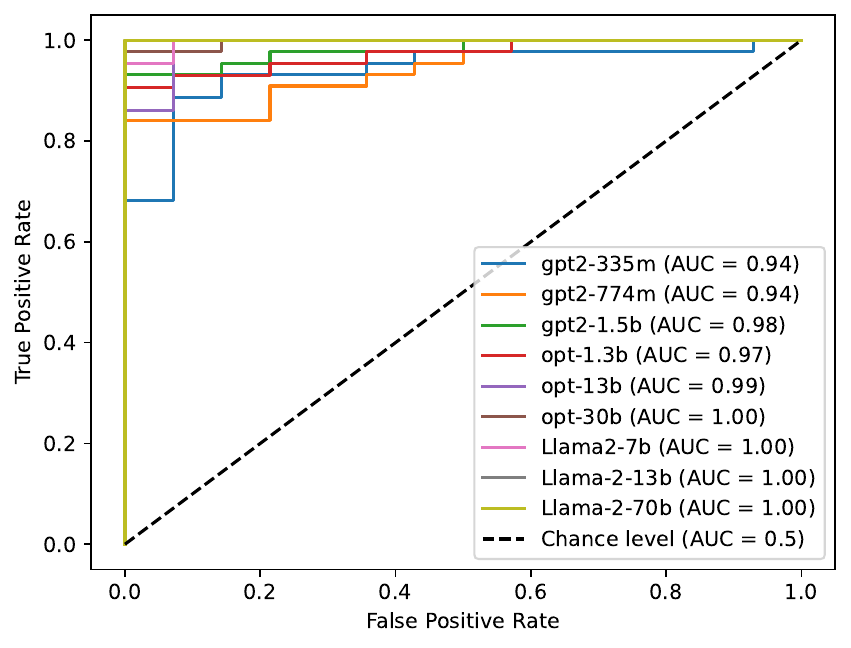}
    \caption{Receiver Operating Characteristic (ROC) of template tokens and content tokens with \textbf{another template}.}
    \label{fig:ROC2}
    \vspace{-12pt}
\end{figure*}

\subsection{Variance-based T/C classification}
\label{app:more_autoregressive_classify}

Here we exhibit some details of our variance-based T/C classifier.

\subsubsection{Word-level analyzer}
\label{app:word-level}
Because of the sub-word level tokenizer used by most LLMs, the replaced content could be divided into different numbers of tokens and affect our dataset alignment. So we need combine several sub-word tokens into a ``word''.
In practice, we found that GPT-4 uses the coarsest-grained tokenizer. So, we first use GPT-4's tokenizer to divide sentences into different tokens. For each token, we check whether its first position is a whitespace or some punctuation, including periods, commas, colons, and semicolons. If it is, we keep it as a single token; otherwise, we concatenate it with the preceding token. This process gives us a word-level analyzer, guaranteeing that each token (generated by a tokenizer of a practical model such as Llama-2 or GPT-3) corresponds to a single ``word'' (defined by our word-level analyzer). When we say ``word'', we always mean a token split by our word-level analyzer.

\subsubsection{Algorithm}
\label{app:algorithm}
The pseudo-code of the autoregressive T-C Classifier based on variance is shown in Algorithm~\ref{alg:TCClassifier}.
\begin{algorithm}[ht]
\caption{Autoregressive T-C Classifier based on variance}
\label{alg:TCClassifier}
\begin{algorithmic}[1] 
\State {\bfseries Input:} sentence $s$ with prompt $p$ as the beginning part, autoregressive model $M$, threshold $\theta$, replacing times $N$
\State Initialize empty lists: $T$, $C$
\State Initialize empty set for sequences with content replacement: $S\leftarrow((),\dots,())$ ($N$ empty sequences.)
\For{each word $w$ in prompt $p$}
\State \textbf{Manually} determine whether $w$ belongs to $T$ or $C$.
\For{$s'\in S$}
    \If{$w$ belongs to $T$}
        \State $s'\leftarrow s'+w$
    \Else
        \State \textbf{Manually} write the replaced content $replaceContent$
        \State $s'\leftarrow s'+replaceContent$
    \EndIf
\EndFor
\EndFor
        
\State $s \leftarrow s - p$

\Repeat
    \State $w\leftarrow s[0]$, $P\leftarrow\{\}$
    \For{each sentences $s'$ in $S$}
        \State Record the output distribution: $P.\text{add}(\mathcal{M}(s'))$
    \EndFor
    \State $variance \leftarrow$ measure the variance in the distribution set $P$.
    
    \If{$variance > \theta$}
        \State Classify $w$ as content: $C.\text{add}(w)$
        \For{each sentences $s'$ in $S$}
            \State $replaceContent\leftarrow \text{ArgMax}(\mathcal{M}(s'))$
            \State $s'\leftarrow s+replaceContent'$
        \EndFor
    \Else
        \State Classify $w$ as template: $T.\text{add}(w)$
        \State Add $w$ to each sentences in $S$.
    \EndIf
    \State $s\leftarrow s[1:]$
\Until{$s$ is empty.}

\State {\bfseries Output:}$T$, $C$
\end{algorithmic}
\end{algorithm}

\subsubsection{First-token-based classification}
We measure variance only on the first token for the T/C classification, and use it to represent the entire word. In most cases, determining the start of a word is sufficient to predict the entire word generation. However, in some cases, this may lead to the incorrect identification of $C$ as $T$. For example, when the generated content includes a pair of quotation marks and the tokenizer treats a single quotation mark as a separate token, the model recognizes the word as content (so the variance should have been higher) while the generation at the first position is fixed (a quotation mark) and the variance is low. It should be noted that simply splitting it into different words might violate the sub-word assumption mentioned above. Therefore, we choose to filter out some meaningless tokens from the output probabilities. Specifically, in the concatenate-last-letter dataset, we remove single whitespaces, line breaks \texttt{\textbackslash n}, and a space-prefixed left quotation mark. In the SingleEQ dataset, we additionally remove space-prefixed dollar signs $\$$ which are used to represent numbers in some training samples. For these tokens to be removed, we consider two methods. The first is to simply set their value to a very small number and then re-normalize through the softmax function. The second is to continue considering the generation at the next position: we first set the probability (after softmax) of these filtered tokens as zero. Then we just set the current position's output as one of the tokens to be removed and get the generation distribution at the next position. We multiply this output's probability distribution by the probability of this token in the original output and then add it to the original distribution. For example, if we want to remove whitespace's probability from the current probability, denoted as $P$, we first set its probability to zero, denoted it as $\hat{P}$. And then we assume the current output is just the whitespace and make the model generate the output probability at the next position, denoted as $P'$. Then we multiply this distribution by the probability of the whitespace $P(\text{whitespace})\cdot P'$ and add it to $\hat{P}$. It is easy to check the sum of the new distribution is still one. This procession is just like skipping the whitespace and replacing the next token in the current place. 
The latter method is more accurate in practice, preventing the filtered tokens from dominating the probabilities, but it significantly reduces inference speed. We only use the latter method in ``SingleEQ'' and we also set a threshold ($1\%$) to pass this procession when the probability of the filtered token does not exceed a certain value.

\subsubsection{Content generation}
Another issue is content generation. Since we no longer use human annotations and instead rely on the model's judgment of when to generate content, we use the model itself to generate tokens that should be filled in under different content replacements. The challenge here is how to determine whether content generation has finished. For example, in the concatenate-last-letters dataset, when the model detects the next token after ``3. The last letter of'' should be content, the model should generate some tokens that will be used to fill in this position while detecting the subsequent positions. However, the generated \texttt{<word3>} could be tokenized into several tokens and cannot be generated in one step. So we need a criterion to detect whether the generation has finished.
Here, we assume one content consists of a \textbf{single} word. We use the same criterion as our word-level analyzer mentioned above, which is to keep adding generated tokens until a token with a space (or punctuation, line break, etc.) is generated. At this point, all previously generated tokens can be combined into a ``word'' and used as a replacement for the current content.

\subsubsection{More results}
\label{app:more_results_of_tc_classification}
More results of the T/C classification are shown in Figure~\ref{fig:more_tc_1}, \ref{fig:more_tc_2} and \ref{fig:more_tc_3}.
\begin{figure}[t]
    \centering
    \includegraphics[width=0.65\linewidth]{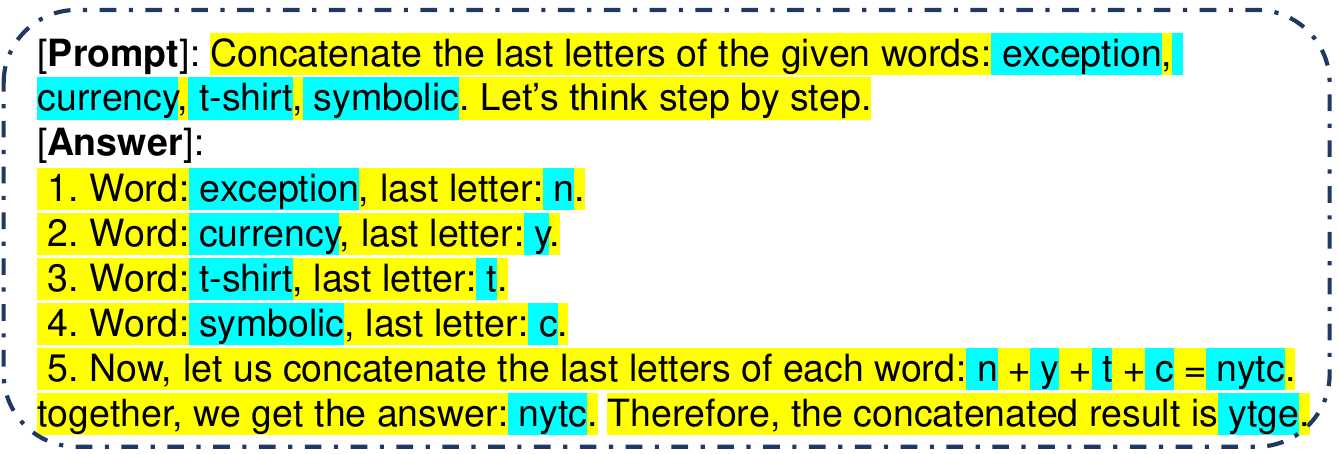}
    \caption{The classification results of the concatenate-last-letter dataset. Threshold: 0.4. Template: \hl{yellow}, content: \sethlcolor{lightblue}\hl{blue}. We mark the token whose classification conflicts with the human intuition as \textcolor{red}{red}.}
    \label{fig:more_tc_1}
\end{figure}
\begin{figure}[t]
    \centering
    \includegraphics[width=0.65\linewidth]{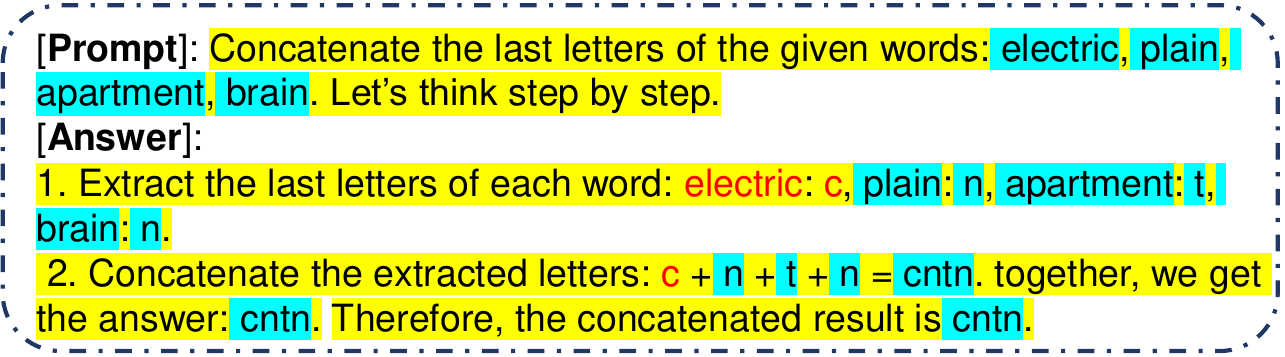}
    \caption{The classification results of the concatenate-last-letter dataset. Threshold: 0.3. Template: \hl{yellow}, content: \sethlcolor{lightblue}\hl{blue}. We mark the token whose classification conflicts with the human intuition as \textcolor{red}{red}.}
    \label{fig:more_tc_2}
\end{figure}
\begin{figure}[t]
    \centering
    \includegraphics[width=0.65\linewidth]{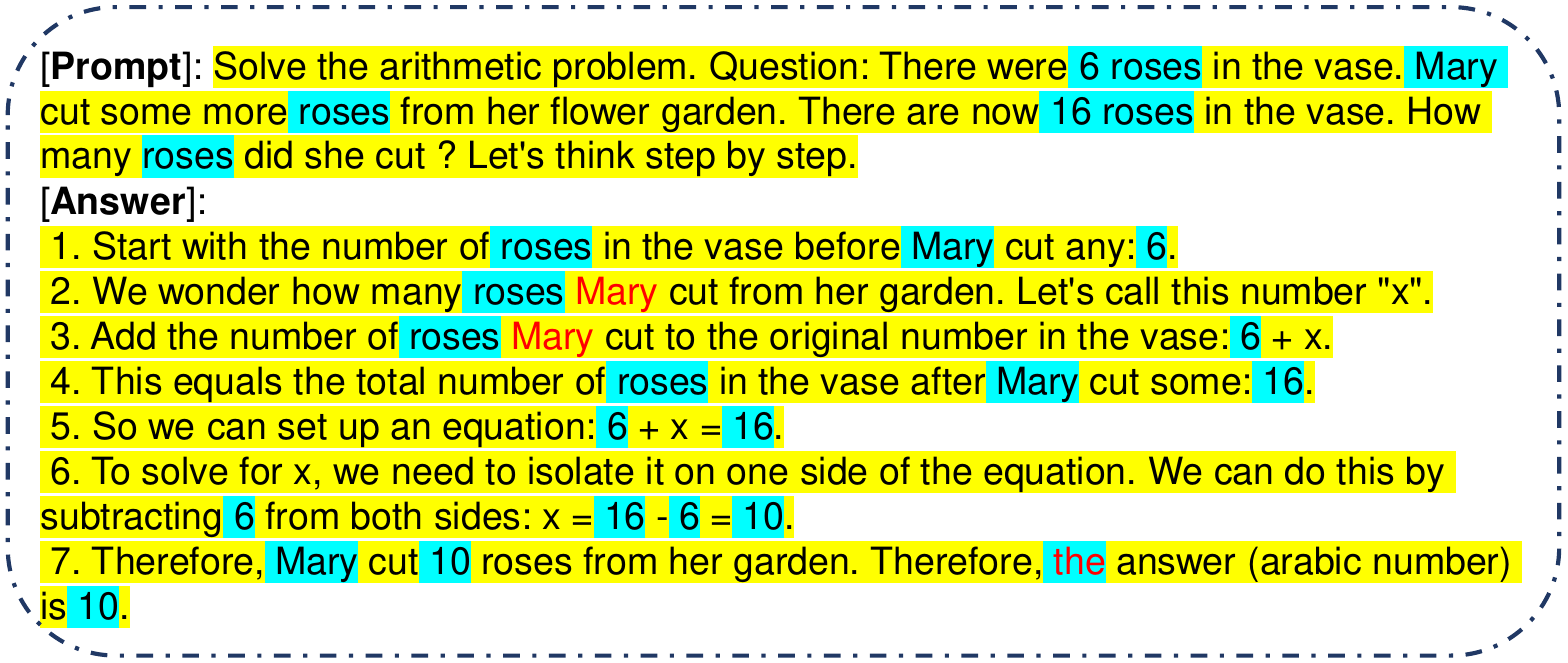}
    \caption{The classification results of the SingleEQ dataset. Threshold: 0.35. Template: \hl{yellow}, content: \sethlcolor{lightblue}\hl{blue}. We mark the token whose classification conflicts with the human intuition as \textcolor{red}{red}.}
    \label{fig:more_tc_3}
\end{figure}

\subsubsection{Error Analysis}
We find some ``misclassification'' actually contains more inspiration than just ``error'', such as the ``habit'' of an LM.
For example, in the Figure~\ref{fig:classifier} Right, the LMs seems to make mistakes on the word ``John'', ``to'' and ``Sam''. Recall that our classification model cannot see the current word but only the previous words to classify the current position to be generated. When we check the detailed logits of LMs, we find the misclassification is because (1) the LM tends to generate a ``Therefore'' at the beginning of the sentence, so it expect the first token should be a template token; (2) the LM generates ``gave Sam 43 seashells.'' instead of the given ``gave 43 seashells to Sam'', where the two phrase has the same meaning but different T/C classification (TCCCT / TCCTC), which leads to misclassification on the last two words.
Similarly, in Figure~\ref{fig:more_tc_3}, the LM tends to generate a ``which'' before the content ``Mary''. In a word, we find the misclassification always occurs in some cases where there are several equal expression with different T/C classification results, and these misclassification is actually due to different language usage between the generation habit of the LM and our groundtruth answer. 

\end{document}